\documentclass{article}

\usepackage{main}

\usepackage[utf8]{inputenc} 
\usepackage[T1]{fontenc}    
\usepackage{hyperref}       
\usepackage{url}            
\usepackage{booktabs}       
\usepackage{amsfonts}       
\usepackage{nicefrac}       
\usepackage{microtype}      
\usepackage{lipsum}
\usepackage{fancyhdr}       
\usepackage{graphicx}       
\graphicspath{{media/}}     

\usepackage{amsmath,amsfonts,bm}

\def\eqref#1{equation~\ref{#1}}

\def\1{\bm{1}}

\DeclareMathAlphabet{\mathsfit}{\encodingdefault}{\sfdefault}{m}{sl}
\SetMathAlphabet{\mathsfit}{bold}{\encodingdefault}{\sfdefault}{bx}{n}

\usepackage{hyperref}
\usepackage{url}

\usepackage{graphicx, caption, subcaption}
\usepackage{stfloats}
\usepackage{algorithm} 
\usepackage{algpseudocode} 
\usepackage{xcolor,colortbl}
\usepackage{tabularray}
\usepackage{colortbl}
\usepackage{wrapfig}
\usepackage{enumitem}
\usepackage{amsmath}
\usepackage{amssymb}
\usepackage{bbm}
\usepackage{soul,xcolor}
\usepackage[utf8]{inputenc} 
\usepackage[T1]{fontenc}    
\usepackage{hyperref}       
\usepackage{url}            
\usepackage{booktabs}       
\usepackage{amsfonts}       
\usepackage{nicefrac}       
\usepackage{microtype}      
\usepackage{xcolor}         
\usepackage{amsthm}
\usepackage{array,multirow,graphicx}
\usepackage{soul}
\usepackage{cancel}
\usepackage{rotating}
\usepackage{pifont}

\pdfoutput=1

\usepackage{algorithm}

\newtheorem{proposition}{Proposition}
\newtheorem{defination}{Definition}
\newtheorem{assumption}{Assumption}
\newtheorem{observation}{Observation}

\newcommand{\model}{\textbf{DeCaf}}
\newcommand{\indep}{\perp \!\!\! \perp}
\newcommand{\rew}[1]{}

\algnewcommand{\IIf}[1]{\State\algorithmicif\ #1\ \algorithmicthen}
\algnewcommand{\EndIIf}{\unskip\ \algorithmicend\ \algorithmicif}

\title{\textbf{DeCaf}: A Causal Decoupling Framework for OOD Generalization on Node Classification}

\author{
  Xiaoxue Han \\
  Stevens Institute of Technology \\
\texttt{xhan26@stevens.edu} \\
  \And
  Huzefa Rangwala \\
 George Mason University\\
  \texttt{rangwala@cs.gmu.edu} \\
   \And
  Yue Ning \\
  Stevens Institute of Technology \\
\texttt{yue.ning@stevens.edu} \\
}

\begin{document}
\maketitle



\newcommand{\fix}{\marginpar{FIX}}
\newcommand{\new}{\marginpar{NEW}}

\newtheorem{property}{Property}[section]

\begin{abstract}
Graph Neural Networks (GNNs) are susceptible to distribution shifts, {creating vulnerability and security issues in critical domains.} There is a pressing need to enhance the generalizability of GNNs on out-of-distribution (OOD) test data. Existing methods that target learning an invariant \textit{(feature, structure)-label} mapping often depend on oversimplified assumptions about the data generation process, which do not adequately reflect the actual dynamics of distribution shifts in graphs. In this paper, we introduce a more realistic graph data generation model using Structural Causal Models (SCMs), allowing us to redefine distribution shifts by pinpointing their origins within the generation process. {Building on this, we propose a \textbf{ca}sual \textbf{de}coupling \textbf{f}ramework, \model{}, that independently learns unbiased \textit{feature-label} and \textit{structure-label} mappings.} We provide a detailed theoretical framework that shows how our approach can effectively mitigate the impact of various distribution shifts. We evaluate \model{} across both real-world and synthetic datasets that demonstrate different patterns of shifts, confirming its efficacy in enhancing the generalizability of GNNs.
\end{abstract}

\section{Introduction}


Graph Neural Networks (GNNs) perform node classification tasks by learning the relationships between node features, local structures, and node labels on a training graph. {They have demonstrated promising performance across various applications, such as social recommendation, traffic forecasting, and chemical property prediction~\cite{Wu2021_GNN_survey}. However, their success largely relies on the in-distribution (ID) assumption~\cite{Liu2023FLOOD}. }
In real life, test samples are often collected from different distributions such as various geographical areas, 
domains, or time periods, leading to potentially different and unknown distributions from the training data.
Consequently, the model might learn an incorrect mapping of \textit{(feature, structure)-label} relationships that fail on test data, leading to a well-known out-of-distribution (OOD) problem~\cite{arjovsky2020invariant}.

One \textit{de facto} approach to graph OOD generalization~\cite{bai2020decaug, Krueger2020, Sagawa2019, Shen2021, Ye2021theoretical} 
assumes that there exist ``true'' correlations between input and labels that are invariant under distribution shifts. 
These correlations can be obtained by identifying the non-causal/spurious parts of the input (\textit{e.g.} sub-ego-networks or subvectors of the node's embedding). 
However, they model the features of a node and its neighborhood as a single entity. By doing so, they implicitly assume that the shifts in features ($\mathbf{X}$) and structures ($\mathcal{A}$)  occur simultaneously and cannot be separated from one another. {A detailed discussion on the limitations of existing methods is provided in Appendix \ref{related-work-ood}.}

However, we claim that such an assumption may fail on many real-world graphs {(\textit{e.g.}, citation graphs, social networks)}. 
The generation process of graph data is complex: for a node, its features and local topology can be viewed as reflections of its true representations by different ``observers''. 
{
For example, to predict whether someone might be prone to drug addiction based on their online social networks, we can analyze their profile features (\textit{e.g.}, job, social status, photos) as the public image they choose to present. The accounts they follow can indicate additional aspects (\textit{e.g.}, hobbies, personal life, mental status) not explicitly stated in their profiles. Both the profile features and their network connections provide insights into the user's true status, but they highlight different facets through distinct mappings. Together, they offer valuable information for predicting the target behavior.
However, these elements may display varying distribution patterns over time or across different locations. For instance, individuals might disclose different profile information depending on state laws, or their network connections might change following updates to the social network's recommendation algorithm. Such variations can occur  separately and both impact the original \textit{(feature, structure)-label} relationships.
}
%
%

We formalize the
graph generation process as a Structural Causal Model (SCM). 
Based on the graph generation process, we redefine the types of distribution shifts. 
Previously, graph OOD largely followed definitions from general data: for \textit{covariate shift}, $P^\text{train}(\mathbf{X}, \mathcal{A})\neq P^\text{test}(\mathbf{X}, \mathcal{A})$; for concept shift, $P^\text{train}(\mathbf{Y}|\mathbf{X}, \mathcal{A})\neq P^\text{test}(\mathbf{Y}|\mathbf{X}, \mathcal{A})$~\cite{gui2022good}. 
We reformulate the definition, attributing \textit{covariate shift} to changes in the true representations, and define two types of \textit{concept shift} based on the different mappings of features and structure, respectively. We justify the universality of our new definitions.

To this end, we demonstrate that any type of distribution shift can alter the \textit{(feature, structure)-label} mapping. However, we observe that because the generating mechanism of the features or structures does not change, the true \textit{feature-label} or \textit{structure-label} mappings should remain invariant. Based on this observation, we propose a causal decoupling framework, \model{}, that learns unbiased \textit{feature-label} or \textit{structure-label} mappings as causal effects for predicting node labels. Intuitively, \model{} answers the question: \textit{``What information about the label can the node features provide when its local structure is unavailable, and vice versa?"} We provide a theoretical analysis to demonstrate the feasibility and effectiveness of \model{}. We also present an implementable paradigm of \model{} that utilizes Generalized Robinson Decomposition to estimate the causal effects as a practical solution. We evaluate our proposed method across both real-world and synthetic datasets with different patterns of shifts. The results consistently demonstrate that our proposed method improves the generalization ability of GNNs on the node classification task.

\section{Casual Decoupling Framework}
In this section, we introduce the causal decoupling framework, \model{}, which performs node classification by independently estimating the treatment effects of node features and neighborhood representations. We develop a comprehensive theoretical foundation to support the rationale: Initially, we present a novel \textbf{graph generation process} using a Structural Causal Model (SCM), which contrasts with the assumptions underlying most GNN models (section~\ref{sec:graph-generation-process}). Utilizing this SCM, we redefine various \textbf{types of distribution shifts in graph data} and analyze how each SCM element changes under these shifts (Section \ref{sec::dist-shift}). This analysis leads to the development of the \textbf{causal decoupling framework}, which seeks to separately assess the direct impacts of node features and neighborhood representations on node labels (section~\ref{sec::graph-decoupling}). To achieve an unbiased estimation of their impact, we treat the impact as a treatment effect, which can be estimated with a \textbf{casual estimation} model that considers the confounder effect (section~\ref{sec::casual-effecr-estimation}). We visually summarize this conceptual flow in Figure \ref{fig::overall}. Following this theoretical basis, we propose \textbf{a practical end-to-end paradigm} that leverages SOTA casual inference technologies to estimate the decoupled representations and make final predictions, as detailed in Section \ref{sec::casual-effecr-estimation}.

\begin{figure}
    \centering   
    \includegraphics[width=0.5\textwidth]{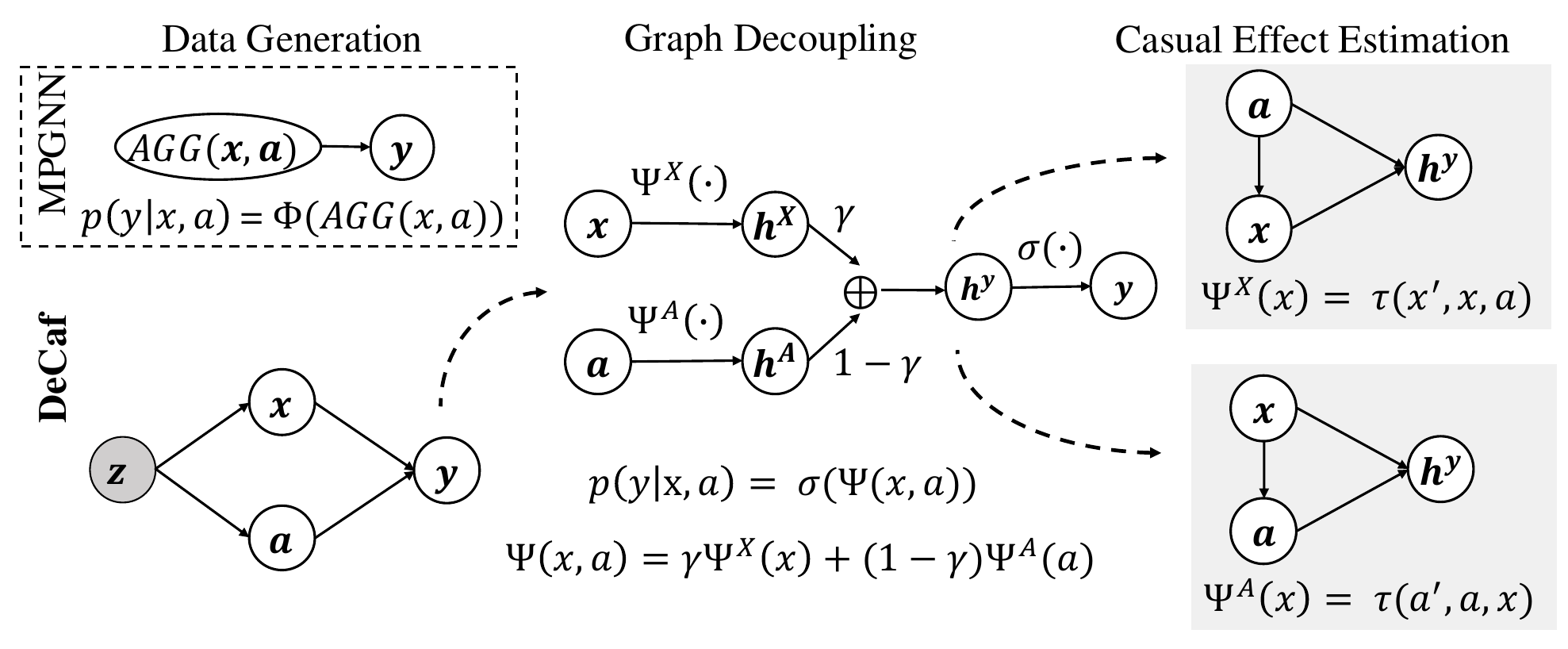}
    \caption{An overview of the conceptual flow. First, we introduce a new data generation process with the SCM. We separately estimate the individual impact of the node features and neighborhood representations on node labels in graph decoupling. To achieve an unbiased estimation of their impact, we propose to treat the impact as a treatment effect, which can be estimated with a casual estimation model that considers the confounding effect.  }
    \label{fig::overall}
\end{figure}
\subsection{Graph generation process}
\label{sec:graph-generation-process}
We build a SCM for graph generation processes based on the assumption that for each node $v_i$, there exists an unobserved raw latent vector $\mathbf{z}_i\in \mathbb{R}^{p}$ that is fully informative about its \textit{``true nature''}, and its ground-truth label, $\mathbf{y}_i\in \mathbb{R}^{k}$, is an affine transformation of $\mathbf{z}_i$ to the label space. 
For a connected graph $\mathcal{G} = (\mathcal{V}, \mathcal{E})$, where $\mathcal{V}$ is the node set and $\mathcal{E}$ is the edge set, we directly observe node features $\mathbf{X}\in \mathbb{R}^{n\times d}$, and the adjacency matrix 
$\mathcal{A}\in {\{0,1\}}^{n\times n}$. Although the latent variable matrix 
$\mathbf{Z}\in \mathbb{R}^{n\times p}$ cannot be directly observed, we believe there exists a ``\textit{hidden observer}'' that can observe $\mathbf{Z}$ and decide $\mathbf{X}$ and $\mathcal{A}$ based on some rules. 

{
\begin{defination}
The ``true nature'', denoted as $\textbf{z}$, is an unobserved latent variable representing ``all facts'' (extrinsic or intrinsic) about a node in a graph. For a node instance $v_i$, knowing $\textbf{z}_i$ provides sufficient information about its features $\textbf{f}_i$ and its connectivity with any other node $u_j$ as $\mathcal{A}_{ij}$. Specifically, there exists a function such that $\mathbf{x}_i = f(\mathbf{z}_i)$ and a function such that $\mathcal{A}_{ij}=g(\mathbf{z}_i, \mathbf{z}_j)$.
\end{defination}
}

{ \textbf{{An intuitive example. }}A person's ``\textit{true nature}'' would encompass not only extrinsic details like gender and education but also intrinsic qualities such as personality and beliefs. These intrinsic features are difficult to measure directly and completely, yet they largely influence a person’s public profile (\textit{e.g.}, personal webpage) and social relationships (\textit{e.g.}, friendships). For example, the content of a personal webpage is shaped not only by true experiences but also by the individual’s personality, which affects how those experiences are presented. Understanding a person’s \textit{true nature} provides a complete view of their behaviors.}

For simplicity of analysis, we assume the features of node $v_i$, 
$\mathbf{x}_i$, is an affine transformation of $z_i$, and $\mathcal{A}_{ij}$ is decided by the similarity between some linear transformations of $\mathbf{z}_i$ and $\mathbf{z}_j$. 

\begin{assumption}
\label{assumption_1}
The generation process of $(\mathbf{X}, \mathbf{Y}, \mathcal{A})$ given $\mathbf{Z}$ can be expressed as follows:
\begin{equation}
\mathbf{x}_i = \mathcal{M}_{f}\mathbf{z}_i + \mathbf{b}_{f}, 
\end{equation}
\begin{equation}
\label{eqn::y-generation}
\mathbf{y}_i = \mathcal{M}_{y}\mathbf{z}_i + \mathbf{b}_{y}, 
\end{equation}
\begin{equation}
p(\mathcal{A}_{ij}=1) = c\cdot\left({\Vert\mathcal{M}_{s}\mathbf{z}_i, \mathcal{M}_{o}\mathbf{z}_j\Vert^{2}_{2} + 1}\right)^{-1}, 
\end{equation}
where $\mathcal{M}_{f}\in \mathbb{Z}^{d\times p}$, $\mathcal{M}_{y}\in \mathbb{Z}^{k\times p}$,
$\mathcal{M}_{s}\in \mathbb{Z}^{q\times p}$, and $\mathcal{M}_{o}\in \mathbb{Z}^{q\times p}$ are constant linear transformation matrices,  $\mathbf{b}_{f}\in \mathbb{R}^{d}$ and $\mathbf{b}_{y}\in \mathbb{R}^{k}$ are constant vectors, $\Vert\cdot\Vert^{2}_{2}$ is euclidean distance, and $0\leq c\leq1$ is a constant number to control the density of the adjacency matrix.  
\end{assumption}

We claim that Assumption \ref{assumption_1} can be generalized to different scenarios. For instance, a homophilous graph would have $\mathcal{M}_{o}$ $\mathcal{M}_{s}$ assigned with the same values as $\mathcal{M}_{y}$; conversely, a heterophilous graph would have $\mathcal{M}_{o}$ and $\mathcal{M}_{s}$ to be opposite with each other; when $\mathcal{M}_{o}$ or $\mathcal{M}_{s}$ are assigned with values close to zeros, the connection will show more randomized behaviors. 

According to Assumption \ref{assumption_1}, node $v_i$'s feature $\mathbf{x}_i$ is only dependent on $\mathbf{z}_i$; however, its connection with other nodes depends on the latent variable $\mathbf{Z}$ of the whole graph, and it seemingly brings \textit{spill-over effects/inferences}, {which occurs when the treatment received by one instance affects the outcome of another instance. This effect breaks \textit{Stable Unit Treatment Value Assumption (SUTVA)}   that the potential outcomes of any unit do not vary with the treatment assigned to other units. It is one of the core assumptions in causal inference as we will discuss in section 2.4.}
To address this, we investigate the correlation between the features of a central node and its neighboring nodes. As the adjacent matrix $\mathcal{A}$ does not contain the features of  neighboring nodes, we define an embedding matrix 
\begin{equation}
    \mathbf{a}_i = \text{agg} (\{ \mathcal{M}_a \mathbf{z}_{j}\}) \,\,\,\, \text{where}\,\,\,\, j\in \mathcal{N}_{i}^{1}\cup\mathcal{N}_{i}^{2}...\mathcal{N}_{i}^{l},
\end{equation}
where $\mathbf{a}_i$ is the $i$-th vector of $\mathcal{A}$ which represents the neighborhood information of all nodes, 
$\mathcal{M}_a$ is a linear transformation matrix, and $\mathcal{N}_{i}^{l}$ is the node set of $l$-hop neighbors of node $v_i$. $\text{agg}(\cdot)$ is an aggregation function (e.g. mean).  
{Although this may introduce undesired correlations between samples that complicate our analysis, we show that under the \textit{Law of Large Numbers}, $\mathbf{a}_i$ experiences negligible spillover effects from other samples. Further details are discussed in Appendix \ref{append:spillover}.} 




To this end, we build a SCM in Figure \ref{fig::scm} (a-\textit{left}) to represent the relationships between the variables $\mathbf{z}$, $\mathbf{x}$, $\mathbf{a}$, and $\mathbf{y}$. 
{When developing a machine learning model, the \textit{Close World Assumption} is generally followed, that the training data encompasses sufficient information to make accurate predictions. Applying that to our case, 
given the unobservability of $\mathbf{z}$, we need to assume that all features within $\mathbf{z}$ pertinent to $\mathbf{y}$ are recoverable through $\mathbf{x}$ and $\mathbf{a}$. 
For the convenience of our narrative, we assume $\mathbf{z}$ is directly caused by $\mathbf{a}$ and $\mathbf{x}$, and we replace the causal link between $\mathbf{z}$ and $\mathbf{y}$ with pseudo-causal links $\mathbf{x} \dashrightarrow \mathbf{y}$ and $\mathbf{a} \dashrightarrow \mathbf{y}$. \textbf{Future analysis is based on the modified SCM in Figure \ref{fig::scm} (a-\textit{right}).}
Based on the modified SCM, the correlation between $\mathbf{x}$ and $\mathbf{y}$ is constituted by two back door paths: $\mathbf{x} \leftarrow \mathbf{z} \rightarrow \mathbf{y}$ and $\mathbf{x} \leftarrow \mathbf{z} \rightarrow \mathbf{a} \rightarrow \mathbf{y}$. Similarly, the correlation between $\mathbf{a}$ and $\mathbf{y}$ is constituted by $\mathbf{a} \leftarrow \mathbf{z} \rightarrow \mathbf{y}$, $\mathbf{a} \leftarrow \mathbf{z} \rightarrow \mathbf{x} \rightarrow \mathbf{y}$. For the paths $\mathbf{x} \leftarrow \mathbf{z} \rightarrow \mathbf{y}$ and $\mathbf{a} \leftarrow \mathbf{z} \rightarrow \mathbf{y}$, $\mathbf{z}$ is the common cause (confounder) of $\mathbf{a}$ and $\mathbf{x}$. 
}

\begin{figure}
    \centering
    \includegraphics[width=0.7\textwidth]{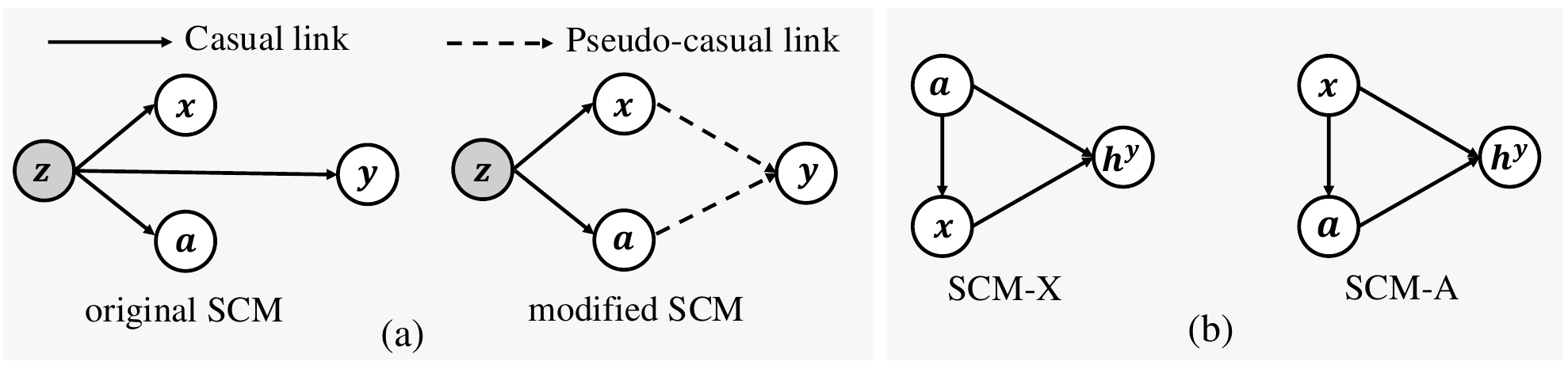}
    \caption{(a) SCMs represent a node's data generation process. (a-\textit{left}) is the original SCM, and (a-\textit{right}) is the modified SCM by replacing the causal link between $z$ and $y$ with pseudo-casual links. (b) The two-view SCMs for the casual effect estimation. For SCM-X, $\mathbf{a}$ is the treatment, and $\mathbf{x}$ is the confounder. For SCM-A, the treatment and the confounder are reversed.  }
    \label{fig::scm}
\end{figure}

\subsection{Types of distribution shifts on graph data}
\label{sec::dist-shift}
Based on the proposed SCM, we redefine the covariate shift and the concept shift on graph data. 
Bearing in mind that OOD generalization is impossible without any assumption in the data generalization process, we make a constraint such that the conditional distribution between the latent variable $\mathbf{z}$ and the ground truth label $\mathbf{y}$ should remain invariant across different domains (e.g. $\mathcal{M}_y^\text{train} = \mathcal{M}_y^\text{test}$, $\mathbf{b}_y^\text{train} = \mathbf{b}_y^\text{test}$or $P^\text{train}(\mathbf{y}|\mathbf{z}) = P^\text{test}(\mathbf{y}|\mathbf{z})$). 
We justify it by pointing out that both $\mathbf{z}$ and $\mathbf{y}$ characterize the intrinsic nature of the data point, and the invariance of their correlation is necessary to ensure that OOD generalization is possible. 
Starting with that, we identify different sources of distribution shifts. 
We attribute covariate shift to the drift of the latent variable $\mathbf{z}$.
\begin{defination}
\textit{(Covariate Shift)}: $P^\text{train}(\mathbf{z}) \neq P^\text{test}(\mathbf{z})$ while $\mathcal{M}_x, \mathbf{b}_x, \mathcal{M}_s, \mathbf{b}_s, \mathcal{M}_o, \mathbf{b}_o$ remain the same for training and test data.
\label{def::convariate-shift}
\end{defination}
We attribute concept shift to the changes in the generation process of node features or the edges with a fixed $\mathbf{z}$ distribution. We define each case, separately.


\begin{defination}
\textit{(Concept Shift-X)}:  $P^\text{train}(\mathbf{z}) = P^\text{test}(\mathbf{z})$ while $(\mathcal{M}_x^\text{train}, \mathbf{b}_x^\text{train}) \neq (\mathcal{M}_x^\text{test}, \mathbf{b}_x^\text{test})$ and the rest parameters remain the same.
\end{defination}

\begin{defination}
\textit{(Concept Shift-A)}: $P^\text{train}(\mathbf{z}) = P^\text{test}(\mathbf{z})$ while $(\mathcal{M}_s^\text{train}, \mathbf{b}_s^\text{train}) \neq (\mathcal{M}_s^\text{test},\mathbf{b}_s^\text{test})$ and the rest parameters remain the same.
\end{defination}

We investigate and summarize the behavior of the joint distribution of $\mathbf{x}$ and $\mathbf{a}$ and the relations between other variables under different distribution shifts in Appendix \ref{appendix-distribution-shifts}. As it shows, under all three types of distribution shifts, the distribution $p(\mathbf{y}|\mathbf{a}, \mathbf{x})$ changes, thus the directly estimated correlation between $\mathbf{y}$ and $\mathbf{a}, \mathbf{x}$ with the training set may fail in the test set. On the other hand, at least one of $p(\mathbf{y}|\mathbf{x})$ and $p(\mathbf{y}|\mathbf{a})$ remains constant under the shifts.

\subsection{Graph decoupling}
\label{sec::graph-decoupling}
We propose to estimate $p(\mathbf{y}|\mathbf{a}, \mathbf{x})$ as the combination of $p(\mathbf{y}|\mathbf{x})$ and $p(\mathbf{y}|\mathbf{a})$. However, when predicting with a graph neural network, $p(\mathbf{y}|\mathbf{a}, \mathbf{x})$ is often modeled as $\sigma(\mathbf{h}_{y}) = \sigma (\Psi(\mathbf{a}, \mathbf{x}))$, where $\mathbf{h}_{y} = \Psi(\cdot)$ represents the output embedding of the GNN model before the activation, and $\sigma(\cdot)$ is the non-linear activation function. Due to the non-linearity of $\sigma(\cdot)$, the assumption that $p(\mathbf{y}|\mathbf{a}, \mathbf{x})$ can be estimated as the combination of $p(\mathbf{y}|\mathbf{x})$ and $p(\mathbf{y}|\mathbf{a})$ can be easily violated. Also, the mapping to the probability space before the combination could potentially lose the rich information delivered by the embedding space. To address this, we perform the combination process in the embedding space instead. 

\begin{assumption}
\label{assumption-2}
The value of 
 $\Psi(\mathbf{x}, \mathbf{a})$ can be estimated as the weighted average of the values of individual functions $\Psi_{\mathbf{x}}(\mathbf{x})$ and $\Psi_{\mathbf{a}}(\mathbf{a})$.
\begin{equation}
\label{eq::additive}
    \Psi(\mathbf{x}, \mathbf{a}) = \gamma \cdot \Psi_{\mathbf{x}}(\mathbf{x}) + (1-\gamma) \cdot \Psi_{\mathbf{a}}(\mathbf{a}), 
\end{equation}
where $0\leq \gamma \leq 1$ is a constant hyperparamter {that controls the ratio between the contribution of $\Psi_{\mathbf{x}}(\mathbf{x})$ and $\Psi_{\mathbf{a}}(\mathbf{a})$ to the final prediction. } 
\end{assumption}

We claim that Assumption~\ref{assumption-2} is a reasonable assumption when $\Psi(\mathbf{x}, \mathbf{a})$ is modeled with message-passing graph neural networks (MPGNNs). Most MPGNNs obtain the representation of the central node by aggregating its neighboring nodes with operations like summing or (weighted) averaging, which can be viewed as 
a process in which each central/neighboring node shares its ``vote'' on deciding the outcome. 
To better make our point, we analyze classification on the Simple Graph Convolution (SGC) model~\cite{wu2019-SGC}. We show that when a fixed weight $\gamma$ is assigned to the central node during the aggregation process, the prediction function of SGC can be written as:

\begin{align}
\begin{split}
\sigma\Bigl(\gamma\underbrace{[({{\tilde{D}}^{k}})^{-\frac{1}
{2}}A\tilde{A}^k({{\tilde{D}}^{k}})^{\frac{1}{2}}\mathbf{X}\Theta]_i}_{\Psi_{\mathbf{a}}(\mathbf{a}_i)} \\
+ (1-\gamma)\underbrace{[(({{\tilde{D}}^{k}})^{-\frac{1}
{2}}I({{\tilde{D}}^{k}})^{\frac{1}{2}})^{k}\mathbf{X}\Theta]_i}_{\Psi_{\mathbf{x}}(\mathbf{x_i})}\Bigr), 
\end{split}
\end{align}

where $\hat{y} = \sigma ({S}^{k}\mathbf{X}\Theta)$, $S$ is the
“normalized” adjacency matrix $S = \tilde{D}^{-\frac{1}{2}}\tilde{A}\tilde{D}^{\frac{1}{2}}$, $\tilde{A} = A+I$, $\tilde{D}$ is the degree matrix of $\tilde{A}$, $k$ is the number of layers, $\Theta$ is the parameterized weights of each layer into a single matrix: $\Theta = \Theta_{0}\Theta_{1}...\Theta_{k}$.  We can view  $[({{\tilde{D}}^{k}})^{-\frac{1}
{2}}A\tilde{A}^k({{\tilde{D}}^{k}})^{\frac{1}{2}}\mathbf{X}\Theta]_i$ as $\Psi_{\mathbf{a}}(\mathbf{a}_i)$, and $[(({{\tilde{D}}^{k}})^{-\frac{1}
{2}}I({{\tilde{D}}^{k}})^{\frac{1}{2}})^{k}\mathbf{X}\Theta]_i$ as $\Psi_{\mathbf{x}}(\mathbf{x_i})$, which aligns with asssumtion \ref{assumption-2}. {Notice for $\Psi_{\mathbf{a}}(\mathbf{a}_i)$ the adjacency matrix at the first layer does not include self-loops, ensuring that the node feature of the instance itself is not included.}

We thus claim that Assumption 2 is reasonable. Note that it only holds when $\Psi_{\mathbf{x}}(\cdot)$ and $\Psi_{\mathbf{a}}(\cdot)$ are unbiased estimations. As shown in Figure \ref{fig::scm} (a), $\mathbf{a}$ and $\mathbf{x}$ are caused by a common factor $\mathbf{z}$, thus they act like confounders of each other through back-door paths  $\mathbf{x} \leftarrow \mathbf{z} \rightarrow \mathbf{a} \rightarrow \mathbf{y}$ and $\mathbf{a} \leftarrow \mathbf{z} \rightarrow \mathbf{x} \rightarrow \mathbf{y}$. The direct estimation of $\Psi(\mathbf{x})$ and $\Psi(\mathbf{a})$ is biased. In the next section, we aim to obtain unbiased estimations of $\Psi_{\mathbf{x}}(\cdot)$ and $\Psi_{\mathbf{a}}(\cdot)$ by considering the confounder effect.

\subsection{Casual effect estimation}
\label{sec::casual-effecr-estimation}

To clearly understand how changing $\mathbf{a}$/$\mathbf{x}$ directly influences $\mathbf{h}^y$ while considering and adjusting for other confounding factors, we propose to treat $\Psi_{\mathbf{x}}(\mathbf{x})$ and $\Psi_{\mathbf{a}}(\mathbf{a_i})$ as treatment effect of $\mathbf{x}$ and $\mathbf{a}$ on the outcome $\mathbf{h}^y$ (Section.~\ref{sec::graph-decoupling}). 
In a generalized causal effect estimation framework where the treatment can be a continuous representation instead of binary values, we can interpret the effect of treatment $\mathbf{t}$ as \textit{``what additional information $\mathbf{t}$ can provide on predicting 
the outcome''}. In that sense, we want to estimate the treatment effect of $\mathbf{x}$ and $\mathbf{a}$, separately, so each can provide information about the output representation when conditioned on one another. We build two Structural Casual Models to represent the cases where one of $\mathbf{x}$/$\mathbf{a}$ is the treatment and another is the confounder, as shown in Figure \ref{fig::scm} (b). Based on the SCMs, we aim to utilize casual effect inference to estimate the following Conditional Average Treatment Effects (CATEs):
\begin{align}
\begin{split}
    \Psi_{\mathbf{a}}(\mathbf{a}) \triangleq  &\tau (\mathbf{a}', \mathbf{a}, \mathbf{x}) \\
    =& \mathbb{E}[\mathbf{h}^y|C=\mathbf{x}, do(T=\mathbf{a})] \\-& 
    \mathbb{E}[\mathbf{h}^y|C=\mathbf{x}, do(T=\mathbf{a}')], 
    \end{split}
\label{eq:tau_a}
\end{align}

\begin{align}
\begin{split}
    \Psi_{\mathbf{x}}(\mathbf{x}) \triangleq & \tau (\mathbf{x}', \mathbf{x}, \mathbf{a}) \\=& \mathbb{E}[\mathbf{h}^y|C=\mathbf{a}, do(T=\mathbf{x})] \\ -& 
    \mathbb{E}[\mathbf{h}^y|C=\mathbf{a}, do(T=\mathbf{x}')],
    \end{split}
    \label{eq:tau_x}
\end{align}
where $\mathbf{x}'$ and $\mathbf{a}'$ are counterfactual node features and neighborhood representations. Their definitions will be discussed later. 
Since the counterfactual outcome is unobservable,
we follow the common practice and made the assumptions in Appendix \ref{appendix-CATE-assumptions} to estimate the CATEs.

We then propose a 
practical solution
to estimate $\tau (\mathbf{a}', \mathbf{a}, \mathbf{x})$ and $\tau (\mathbf{x}', \mathbf{x}, \mathbf{a})$, and combine them to make predictions.
We apply \textit{Generalized Robinson Decomposition (GRD)}
to isolate the causal estimands and reduce the biases when estimating CATEs.   

\textit{Generalized Robinson Decomposition.}
GRD is proposed as a generalized version of Robinson Decomposition \cite{Robinson1988} to adapt graph-structured treatment. Specifically, GRD assumes that the causal effect is a product effect:

\begin{assumption}
\label{assumption::product-effect}
(Product Effect) We consider the following partial parameterization of $p(y|\mathbf{c},\mathbf{t})$: 
\begin{equation}
\label{equation::product-effect}
    y = g(\mathbf{c})^{\top}h(\mathbf{t}) + \varepsilon,
\end{equation}
where $\mathbf{t}$, $\mathbf{c}$ are the representation of the treatment and the confounder; $g: \mathcal{C} \rightarrow \mathbb{R}^d$, $h: \mathcal{T} \rightarrow \mathbb{R}^d$ and $\mathbb{E}[\varepsilon|\mathbf{c}, \mathbf{t}] = \mathbb{E}[\varepsilon|\mathbf{c}] = 0$, for all $(\mathbf{c}, \mathbf{t})\in \mathcal{C}\times \mathcal{T}$. 
\end{assumption}
{Assumption \ref{assumption::product-effect}  has been proven to be mild and can approximate any arbitrary bounded continuous functions with a small error bound~\cite{kaddour2021causal}.} 
We define \textit{propensity features} as $e(\mathbf{c}) \triangleq \mathbb{E}[h(\mathbf{t})|\mathbf{c}]$ and $m(\mathbf{c})\triangleq \mathbb{E}[\mathbf{y}|\mathbf{c}] = g(\mathbf{c})^{\top}e(\mathbf{c})$. Following the same steps as in Robinson Decomposition, the GRD of Equation \ref{equation::product-effect} is:
$\mathbf{y} - m(\mathbf{c}) = g(\mathbf{c})^{\top}(h(\mathbf{t})-e(\mathbf{c})) + \varepsilon$.
Given nuisance estimates $\hat{m}(\cdot)$ and $\hat{e}(\cdot)$, $g(\cdot)$ and $h(\cdot)$ can be derived with the optimization problem:
\begin{align}
\label{equation::optimization}
\begin{split}
    &\hat{g}(\cdot), \hat{h}(\cdot) \triangleq \arg \min_{g,h} \\
    & \left\{ \frac{1}{n} \sum^{n}_{i=1} \left( \mathbf{y}_{i} - \hat{m}(\mathbf{c}_i)-g(\mathbf{c}_i)^{\top}\bigl(h(\mathbf{t}_{i})-\hat{e}(\mathbf{c}_{i})\bigr) \right) ^{2} 
    \right\}
    \end{split}
\end{align}
With estimated $g(\cdot)$ and $h(\cdot)$, 
the CATE of treatment variable $\mathbf{t}$ and its counterfactual $\mathbf{t}'$ given the confounder $\mathbf{c}$ can be simplified as:
\begin{equation}
\label{equation::CATE-estimation}
    \tau (\mathbf{t}', \mathbf{t}, \mathbf{c}) = g(\mathbf{c})^{\top} \left( h(\mathbf{t}') - h(\mathbf{t})\right). 
\end{equation}
 To optimize Equation \ref{equation::optimization},  we can use \textit{Structured Intervention Networks} (SIN)~\cite{kaddour2021causal}, a two-stage training algorithm to learn $g(\cdot)$ and $h(\cdot)$ as neural networks and estimate the CATE with Equation \ref{equation::CATE-estimation}. In our case, we need to estimate separate set decomposition functions,  $g(\cdot)$ and $h(\cdot)$, for each casual model, and to apply SIN directly would bring the following drawbacks: 1) By separately applying SIN to our two casual models, it would fail to share the representations of the common factors (e.g. the confounder of one model is the treatment of another); the lack of knowledge sharing could make the learning process less efficient and the learned model prone to overfitting. 2) Unlikely common scenarios where counterfactual treatments are well-defined, in our case, $\mathbf{f'}$ and $\mathbf{a'}$ can not be easily decided. Since they are both continuous values that can span the space,  \textit{how to define their values in the situation where the node is \textbf{not} treated by the treatment (node features/neighborhood representations)?} To address these two problems, we propose \textit{Dual Casual Decomposition} and \textit{Background Counterfactual Selection} as components of our method.

\subsubsection{Dual casual decomposition}
We aim to learn two sets of $g(\cdot), h(\cdot)$ for SCM-X and SCM-A. We denote them as $g^{X}(\cdot), h^{X}(\cdot)$ and $g^{A}(\cdot), h^{A}(\cdot)$, separately. Each casual model is also associated with a set of $e(\cdot), m(\cdot)$, denoted as $e^{X}(\cdot), m^{X}(\cdot)$ and $e^{A}(\cdot), m^{A}(\cdot)$.
Observing that directly applying SIN separately could double the training time and be inefficient, we notice that for our two SCMs, the treatment of one model is the confounder of the other. Leveraging this fact, we allow the embedding of the same entity to be shared across the two models. We propose a new paradigm, named {Dual Causal Decomposition}, which first learns the common embedding of the two models and then applies a lighter dual version of SIN to estimate the remaining embedding. Our approach effectively reduces model parameters.
The complete process is provided in Appendix \ref{appendix:dual}.

\subsubsection{Background counterfactual selection}\label{sec:counterfact}

We view the treatment effect as useful information provided by the treatment variables for decision-making in predictions. At each of the casual models, the factual treatment ($\mathbf{a}$ or $\mathbf{x}$) received by an instance reveals information about its label. In the counterfactual ``\textit{untreated}'' case, the treatment representation should reveal no such information. It is problematic to simply set the counterfactual representation as zeros since an all-zeros embedding does not necessarily mean the absence of information. Instead, for each instance, at each time, we randomly sample a treatment representation from the whole dataset, and we answer the question: \textit{what net effect does the factual treatment bring compared to the random counterfactual treatment?} We repeat the above process multiple times and average the net effects as the estimated treatment effect. 

Specifically, for SCM-A,  we randomly sample $k$ neighborhood representations with indexes $s_1, ..., s_k$ from the datasets as the counterfactual treatment. The counterfactual outcome is then estimated as follows: 
\begin{equation}
\mathbf{E}[\mathbf{h}^y|C=\mathbf{a}, do(T=\mathbf{x}')] \triangleq \frac{1}{k} \sum_{i=1}^{k} g^{A}(a_{s_i})^{\top}h^{A}(f_{{s_i}}).
\label{eq:cate-f}
\end{equation}
Similarly, for SCM-X, the counterfactual outcome is estimated as:
\begin{equation}
\mathbf{E}[\mathbf{h}^y|C=\mathbf{x}, do(T=\mathbf{a}')] \triangleq \frac{1}{k} \sum_{i=1}^{k} g^{X}(f_{{s_i}})^{\top}h^{X}(a_{s_i}).
\label{eq:cate-a}
\end{equation}

We then estimate ${\Psi_{\mathbf{a}}}(\mathbf{a})$ and ${\Psi_{\mathbf{x}}}(\mathbf{x})$ with Equation \ref{eq:tau_a} and \ref{eq:tau_x} and make a prediction:
\begin{equation}
    \hat{\mathbf{y}} = \sigma \bigl( \gamma \cdot \Psi_{\mathbf{x}}(\mathbf{x}) + (1-\gamma) \cdot \Psi_{\mathbf{a}}(\mathbf{a})\bigr). 
\label{eq:final_prediction}
\end{equation}
We provide the pseudo-code for \model{} in Algorithm \ref{alg::model}, and a summary of important notations in Table \ref{tab::notation}.

\section{Experiments}\label{sec:exp}

To evaluate the effectiveness of \model{}, we aim to answer the following research questions (\textbf{RQs}):
\textbf{RQ1:}  How well can \model{} handle  \textit{covariate shift}? 
\textbf{RQ2:}  How well can \model{} handle  \textit{concept shift}? 
{\textbf{RQ3:} How does the confounder effect between node feature $\mathbf{x}$ and the neighborhood representation $\mathbf{a}$ impact the performance of different models, and how well can \model{} handle this confounder effect? }

 
\subsection{Comparison methods}\label{sec:comare}
We compare \model{} with SOTA Graph OOD generalization methods that are applicable to the node classification including IRM \cite{arjovsky2020invariant}, REX \cite{Krueger2020}, EERM \cite{Wu2022Handling}, CIT \cite{xia2023learning}, FLOOD \cite{Liu2023FLOOD}, and StableGL \cite{zhang2023StableGL}.  We also compare with empirical risk minimization (ERM) as {a baseline}. 
%
Among these methods, REX \cite{Krueger2020} requires access to multiple training environments, thus it is only applicable to the \texttt{OGB-elliptic} and \texttt{Facebook-100}  datasets with multiple training graphs.
%
%
SR-GNN~\cite{zhu2021shiftrobust} requires access to the input distribution of the test set when training the model and does not apply to the inductive setting studied in this paper. 
{All experiments are conducted on an NVIDIA GeForce RTX 3090 GPU with 24GB memory}. We compare \model{} with the baseline methods regarding their restrictions and complexity in Appendix \ref{sec::compare}. 

 
 

\begin{minipage}[htbp]{\linewidth}%
\centering
\tiny
\captionsetup[table]{hypcap=false}
\captionof{table}{Test Macro-F1 scores on single-graph datasets with soft label-leaveout. ``OOM'' stands for out of memory. The best results are bold-faced.}
\begin{tabular}
{ l l c c c c }

\toprule
\textbf{Dataset} & \textbf{Method} & \textbf{SGC} & \textbf{GCN} & \textbf{GAT} \\
\midrule
\multirow{6}{*}{\texttt{Cora}} 
& ERM    & 65.61$\pm$4.28 & 67.94$\pm$0.89 & 66.95$\pm$2.81 \\
& IRM    & 65.52$\pm$3.01 & 66.04$\pm$2.33 & 65.46$\pm$4.08 \\
& EERM   & 65.27$\pm$2.22 & 67.24$\pm$2.86 & 69.50$\pm$2.66 \\
& CIT    & 60.51$\pm$1.48 & 61.30$\pm$0.72 & 67.80$\pm$2.28 \\
& FLOOD  & 63.58$\pm$4.86 & 62.26$\pm$5.54 & 66.35$\pm$5.43 \\
& StableGL & 59.96$\pm$13.18 & 65.42$\pm$5.20 & 67.28$\pm$1.11 \\
& \model{}  & \textbf{71.41$\pm$1.19} & \textbf{70.12$\pm$1.29} & \textbf{70.58$\pm$0.49} \\
\midrule
\multirow{6}{*}{\texttt{Citeseer}} 
& ERM    & 49.05$\pm$3.15 & 49.58$\pm$1.47 & 52.90$\pm$0.58 \\
& IRM    & 52.98$\pm$1.86 & 51.82$\pm$1.59 & 51.75$\pm$1.00 \\
& EERM   & 44.53$\pm$0.92 & 43.56$\pm$2.10 & 52.63$\pm$4.78 \\
& CIT    & 44.21$\pm$5.75 & 49.36$\pm$0.64 & 55.94$\pm$1.94 \\
& FLOOD  & 46.75$\pm$4.78 & 49.56$\pm$5.49 & 53.08$\pm$0.51 \\
& StableGL & 51.35$\pm$7.34 & 49.62$\pm$3.29 & 51.50$\pm$1.12 \\
& \model{}  & \textbf{59.77$\pm$1.15} & \textbf{58.85$\pm$0.51} & \textbf{56.71$\pm$1.96} \\
\midrule
\multirow{6}{*}{\texttt{Amazon}} 
& ERM    & 86.67$\pm$1.11 & 87.14$\pm$0.33 & 87.71$\pm$0.95 \\
& IRM    & 86.96$\pm$0.14 & 88.01$\pm$0.40 & 86.79$\pm$1.05 \\
& EERM   & 87.63$\pm$0.50 & 88.12$\pm$0.16 & 86.68$\pm$1.43 \\
& CIT    & 87.83$\pm$0.92 & 87.46$\pm$0.91 & 82.28$\pm$4.51 \\
& FLOOD  & 87.08$\pm$0.73 & 87.33$\pm$0.68 & 85.20$\pm$3.72 \\
& StableGL & 87.26$\pm$0.73 & 87.48$\pm$0.45 & 85.26$\pm$2.82 \\
& \model{}  & \textbf{89.67$\pm$0.39} & \textbf{88.93$\pm$0.73} & \textbf{88.74$\pm$0.57} \\
\midrule
\multirow{6}{*}{\texttt{Coauthor}} 
& ERM    & 86.60$\pm$0.91 & 87.66$\pm$0.24 & 80.48$\pm$1.21 \\
& IRM    & 87.96$\pm$0.52 & 88.75$\pm$0.58 & 78.87$\pm$1.20 \\
& EERM   & 84.68$\pm$1.28 & 85.27$\pm$1.04 & OOM \\
& CIT    & 84.48$\pm$0.40 & 86.17$\pm$1.52 & \textbf{85.60$\pm$0.84} \\
& FLOOD  & 83.63$\pm$1.13 & 85.27$\pm$1.04 & OOM \\
& StableGL & 84.48$\pm$0.40 & 86.17$\pm$1.52 & 85.93$\pm$0.76 \\
& \model{}  & \textbf{89.20$\pm$0.37} & \textbf{88.97$\pm$0.26} & 85.28$\pm$1.42 \\
\bottomrule
\end{tabular}
\label{tab::results-1}
\end{minipage}

\subsection{Performance on covariate shift (RQ1)}

We use seven single-graph real-world datasets in which \texttt{Cora}, \texttt{Citeseer}, \texttt{Amazon-photo} and \texttt{Coauthor-CS} are homophilous graphs; \texttt{Squirrel}, \texttt{Roman-empire} and \texttt{Tolokers}
are heterophilous graphs. Further details of the above datasets are provided in Appendix \ref{append:statistics-real-world}. As the above datasets have no clear domain information, we synthetically
create OOD data with \emph{soft label-leaveout}, which is inspired by \emph{label-leaveout} used in OOD detection~\cite{wu2023energybased}.
For OOD generalization, the model is not expected to predict unseen classes,  so instead of completely leaving out partial classes to the test set, we allow the training set to have a small portion of samples from those classes. 
In our experiments, we make sure that the training, validation, and test sets have different class distributions.
{By splitting the samples into groups with different class distributions but a relatively constant \textit{input-label} relationship, }\emph{Soft Label-Leaveout} simulates covariate shift {based on Definition \ref{def::convariate-shift}}.

\noindent

\begin{minipage}[t]{\linewidth}
\tiny
\centering
\captionsetup[table]{hypcap=false}
\captionof{table}{Test F1 scores on heterophilous graphs with soft label-leaveout with H2GCN as backbone.}
\begin{tabular}
{ l cc cc cc}
 \toprule
  & \multicolumn{2}{c}{\texttt{Squirrel}} &  \multicolumn{2}{c}{\texttt{Roman-empire}} & \multicolumn{2}{c}{\texttt{Tolokers}}\\ 
  \midrule
  ERM&\multicolumn{2}{c}{24.12$\pm$3.52}&\multicolumn{2}{c}{42.01$\pm$0.98}&\multicolumn{2}{c}{46.94$\pm$3.07}\\
  IRM&\multicolumn{2}{c}{28.81$\pm$2.06}&\multicolumn{2}{c}{40.58$\pm$0.60}&{47.86$\pm$2.26}\\
  EERM&\multicolumn{2}{c}{30.42$\pm$3.71}&\multicolumn{2}{c}{OOM}&\multicolumn{2}{c}{44.18$\pm$0.18}\\
  CIT&\multicolumn{2}{c}{28.93$\pm$1.80}&\multicolumn{2}{c}{45.41$\pm$3.25}&\multicolumn{2}{c}{44.26$\pm$0.00}\\
  FLOOD&\multicolumn{2}{c}{23.96$\pm$8.73}&\multicolumn{2}{c}{OOM}&\multicolumn{2}{c}{44.58$\pm$0.18}\\
  StableGL&\multicolumn{2}{c}{26.44$\pm$4.73}&\multicolumn{2}{c}{45.78$\pm$2.09}&\multicolumn{2}{c}{44.10$\pm$0.20}\\
\model{}&\multicolumn{2}{c}{\textbf{32.57$\pm$1.70}}&\multicolumn{2}{c}{\textbf{48.85$\pm$0.70}}&\multicolumn{2}{c}{\textbf{60.14$\pm$0.51}}\\
  \bottomrule
\end{tabular}
\label{tab::results-2}
\vspace{1.2em}
\captionof{table}{Test F1 scores on Facebook-100 using GNN with the best validation F1.\newline}
\begin{tabular}
{ l cc cc cc }
 \toprule
 Training&\multicolumn{6}{c}{\texttt{Johns Hopkins} + \texttt{Caltech} + \texttt{Amherst}} \\
 
  Test& \multicolumn{2}{c}{\texttt{Penn}} &  \multicolumn{2}{c}{\texttt{Brown}} & \multicolumn{2}{c}{\texttt{Texas}} \\ 
  \midrule
  ERM&\multicolumn{2}{c}{49.23$\pm$1.72}&\multicolumn{2}{c}{49.68$\pm$0.93}&\multicolumn{2}{c}{48.57$\pm$0.21}\\
  IRM&\multicolumn{2}{c}{35.26$\pm$2.40}&\multicolumn{2}{c}{46.92$\pm$5.66}&\multicolumn{2}{c}{36.86$\pm$1.64}\\
  REX&\multicolumn{2}{c}{44.77$\pm$6.48}&\multicolumn{2}{c}{42.65$\pm$7.34}&\multicolumn{2}{c}{44.05$\pm$8.88}\\
  EERM&\multicolumn{2}{c}{22.62$\pm$22.91}&\multicolumn{2}{c}{49.44$\pm$1.92}&\multicolumn{2}{c}{49.12$\pm$1.71}\\
  CIT&\multicolumn{2}{c}{44.66$\pm$6.65}&\multicolumn{2}{c}{45.26$\pm$6.21}&\multicolumn{2}{c}{42.10$\pm$8.97}\\
   FLOOD&\multicolumn{2}{c}{42.37$\pm$5.06}&\multicolumn{2}{c}{41.48$\pm$5.28}&\multicolumn{2}{c}{40.82$\pm$5.94}\\
    StableGL&\multicolumn{2}{c}{44.54$\pm$6.58}&\multicolumn{2}{c}{45.64$\pm$6.25}&\multicolumn{2}{c}{43.75$\pm$5.65}\\
\model&\multicolumn{2}{c}{\textbf{55.31$\pm$0.40}}&\multicolumn{2}{c}{\textbf{53.31$\pm$0.11}}&\multicolumn{2}{c}{\textbf{53.56$\pm$0.19}}\\
\bottomrule
\end{tabular}
\label{tab::results-facebook}
\end{minipage}

\normalsize
We compare \model{} with the baselines on the four homophilous graph datasets with soft label-leaveouts, \texttt{Cora}, \texttt{Citeseer}, \texttt{Amazon-photo}, and \texttt{Coauthor-CS}, with SGC, GCN, and GAT as GNN backbones. We report the Macro-F1 score in Table \ref{tab::results-1}. We observe that \model{} outperforms ERM across all cases with an average improvement of 4.2\%,
demonstrating its ability to mitigate the negative impact of covariate shifts. Also, \model{} beats the best baselines in most cases with an average improvement of 2.4\%, showing that it better handles covariate shifts. 

We compare \model{} with the baselines on three heterophilous graph datasets with \textit{soft label-leaveouts}, \texttt{Squirrel}, 
\texttt{Roman-empire}, and \texttt{Tolokers}. {Instead of using common GNNs that are not suited for heterophilous graphs \cite{platonov2024critical}, we use H2GCN~\cite{Zhu2020}, a  SOTA GNN well established for graph heterophily problems, as the backbone GNN for these datasets. }
We report the performance of \model{} compared to the baselines in Table \ref{tab::results-2}.
We report Micro F1 score for \texttt{Squirrel}, \texttt{Roman-empire} with multiple classes, and F1 score for \texttt{Tolokers} with binary classes.
{On average, \model{} improves the (Macro-) F1 scores by 5.8\%.} We observe that~\model{} outperforms the baselines among all datasets. 

\begin{figure}
    \centering
    \includegraphics[width=0.35\textwidth]{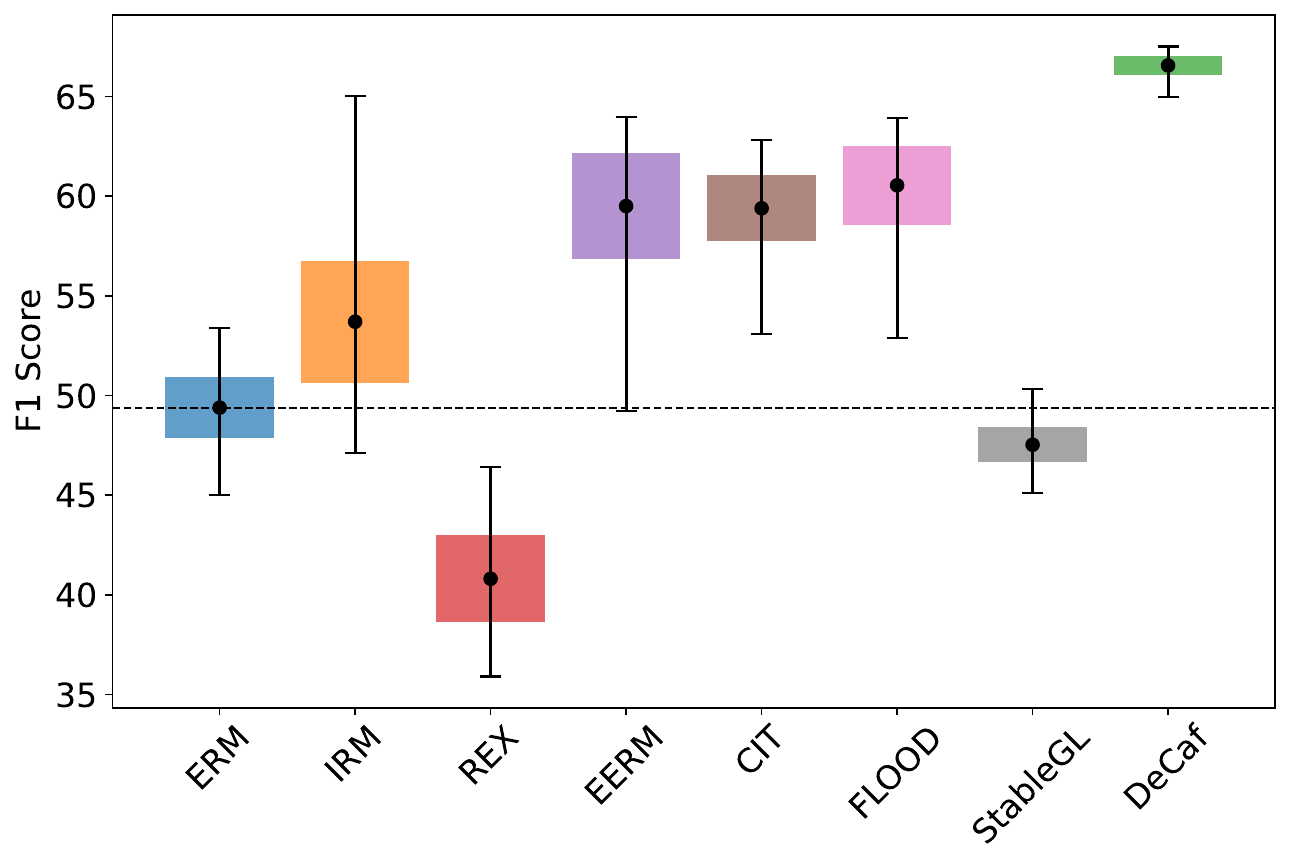}
    \caption{Distribution of F1 scores on \texttt{OGB-elliptic} of different models shown in a bar plot. The dashed line shows the mean F1 score of the ERM method.}
    \label{fig::elliptic}
\end{figure}

\subsection{Performance on concept shift (RQ2)}
\texttt{Facebook-100}~\cite{fb100} contains 100 social networks collected from universities in the United States. 
Following \cite{Wu2022Handling}, we adapt three graphs for training, two graphs for validation, and three graphs for testing. Three sets of training graphs are used. 
\texttt{OGB-elliptic}~\cite{elliplic} is a dynamic financial network dataset that contains  43 graph snapshots from different time steps.
We use the first 5 graph snapshots for training, the next 5 for validation, and the last 33 for testing. Further details of the two datasets are provided in Appendix \ref{append:statistics-real-world}. 

We first identify distribution shifts within both datasets. To assess concept shifts, we employ Hotelling's T-squared statistics to measure the disparity between node distributions sharing the same class label. {Our intuition is that if the nodes from the same class in two graphs have significantly different distributions, we can claim the graphs exhibit a concept shift with different \textit{feature/structure-label} mappings.} We present the pairwise T-squared scores comparing node feature embeddings and neighborhood representations of nodes from the first class across subgraphs for \texttt{Facebook-100} (Figure \ref{fig::t-square} (a)) and \texttt{OGB-elliptic} (Figure \ref{fig::t-square} (a)). A higher T-squared score indicates a greater dissimilarity between the two distributions. 
In Figure \ref{fig::t-square}, we observe that the subgraphs within \texttt{Facebook-100} experience significantly less extent shift in node features compared to neighborhood representations, suggesting a dominance of neighborhood shift (def. concept shift-A). 
For \texttt{OGB-elliptic}, we initially notice a gradual increase in the T-squared score across the axis, indicating that the distribution shift intensifies over time. Additionally, we observe that it experiences concept shift-A and concept shift-X more evenly, with the latter slightly more dominant.

\begin{figure}[htbp]
\centering
\begin{tabular}{cc}
  \includegraphics[height=20mm]{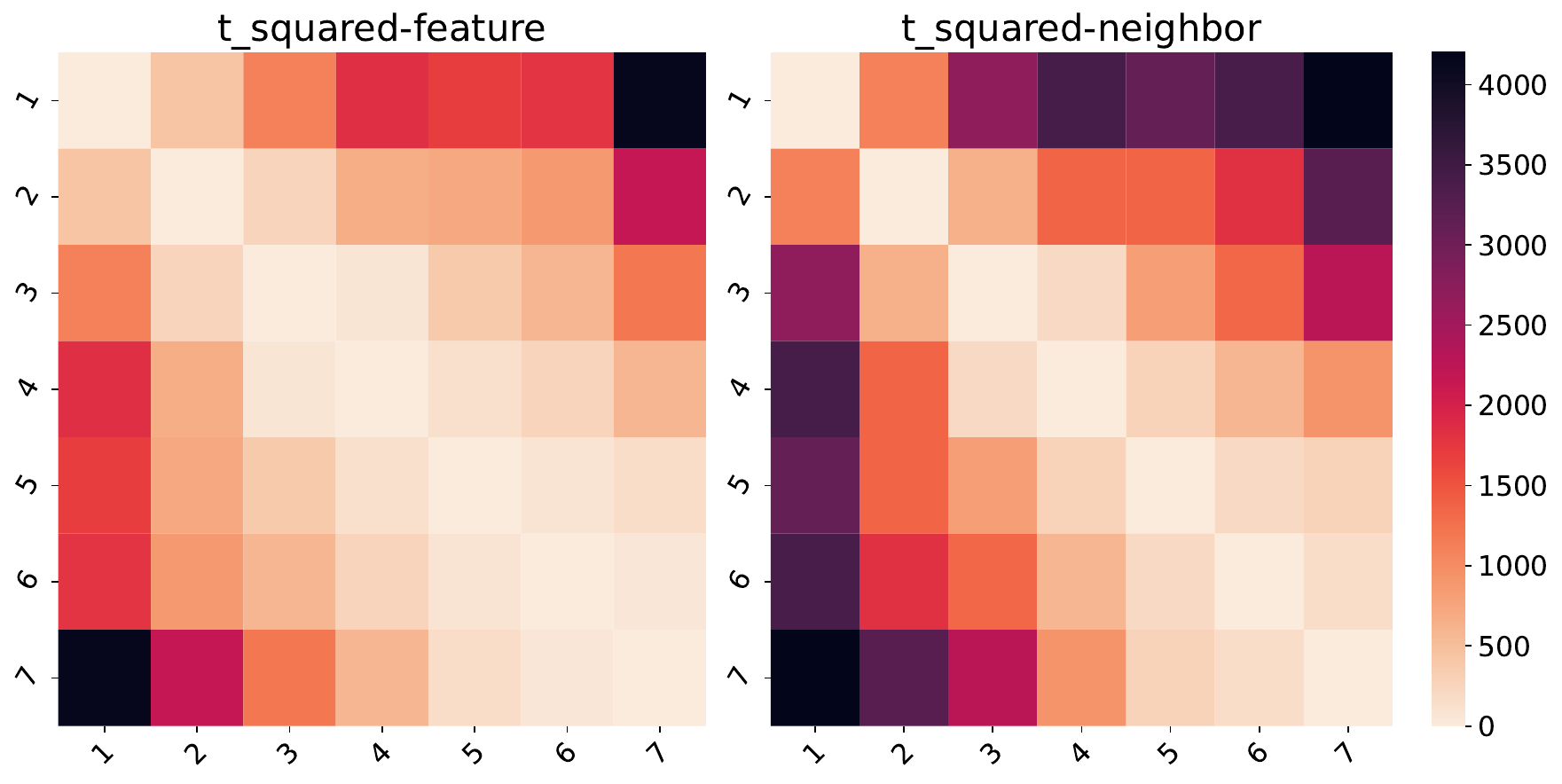}&\includegraphics[height=20mm]{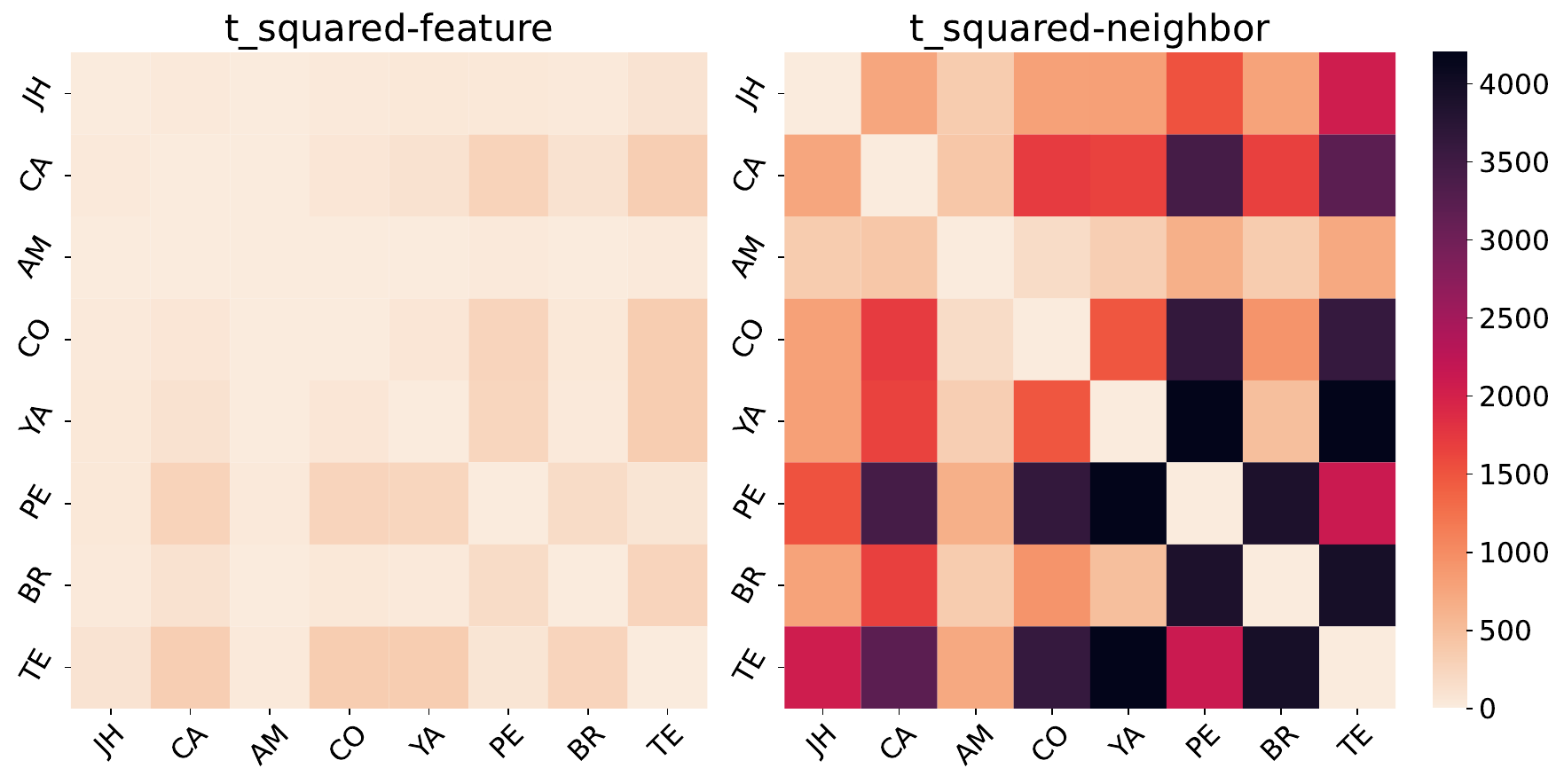} \\
(a) \texttt{OGB-elliptic} & (b) \texttt{Facebook-100}\\[2pt]
\end{tabular}
\caption{Hotelling's two-sample t-squared statistic of node feature embedding between subgraphs from different time periods in \texttt{OGB-elliptic} dataset.} 
\label{fig::t-square}
\end{figure}
\begin{table}
\caption {Test Macro-F1 scores on synthetic datasets. The best results are bold-faced.\newline}
\tiny
\centering
\begin{tabular}
{ l l cc cc cc}
 \toprule
  && \multicolumn{2}{c}{\texttt{h-feat}} &  \multicolumn{2}{c}{\texttt{qrt-feat}} & \multicolumn{2}{c}{\texttt{full-feat}}\\ 
  \midrule
\multirow{7}{*}{\rotatebox[origin=c]{90}{best GNN}}&ERM&\multicolumn{2}{c}{40.78$\pm$11.45}&\multicolumn{2}{c}{21.68$\pm$3.81}&\multicolumn{2}{c}{41.76$\pm$0.63}\\
&IRM&\multicolumn{2}{c}{38.10$\pm$2.54}&\multicolumn{2}{c}{12.86$\pm$4.47}&\multicolumn{2}{c}{13.16$\pm$6.42}\\
&EERM&\multicolumn{2}{c}{35.15$\pm$15.09}&\multicolumn{2}{c}{39.46$\pm$10.48}&\multicolumn{2}{c}{\textbf{54.71$\pm$0.74}}\\
&CIT&\multicolumn{2}{c}{30.83$\pm$25.93}&\multicolumn{2}{c}{18.37$\pm$11.45}&\multicolumn{2}{c}{38.28$\pm$1.70}\\
&FLOOD&\multicolumn{2}{c}{47.22$\pm$4.16}&\multicolumn{2}{c}{35.59$\pm$21.58}&\multicolumn{2}{c}{42.17$\pm$9.55}\\
&StableGL&\multicolumn{2}{c}{32.83$\pm$12.24}&\multicolumn{2}{c}{28.56$\pm$20.37}&\multicolumn{2}{c}{45.12$\pm$0.77}\\
&\model{}&\multicolumn{2}{c}{\textbf{57.24$\pm$13.75}}&\multicolumn{2}{c}{\textbf{49.00$\pm$2.61}}	&\multicolumn{2}{c}{{54.11$\pm$1.14}}\\
  \midrule
   \multirow{7}{*}{\rotatebox[origin=c]{90}{H2GCN}}&ERM&\multicolumn{2}{c}{49.38$\pm$3.44}&\multicolumn{2}{c}{31.72$\pm$1.60}&\multicolumn{2}{c}{66.59$\pm$1.86}\\
  &IRM&\multicolumn{2}{c}{51.17$\pm$8.82}&\multicolumn{2}{c}{33.57$\pm$4.78}&\multicolumn{2}{c}{62.04$\pm$1.45}\\
  &EERM&\multicolumn{2}{c}{49.04$\pm$3.54}&\multicolumn{2}{c}{30.04$\pm$1.72}&\multicolumn{2}{c}{65.67$\pm$1.62}\\
  &CIT&\multicolumn{2}{c}{54.33$\pm$3.90}&\multicolumn{2}{c}{30.28$\pm$3.96}&\multicolumn{2}{c}{64.78$\pm$2.64}\\
  &FLOOD&\multicolumn{2}{c}{47.18$\pm$4.97}&\multicolumn{2}{c}{29.56$\pm$1.13}&\multicolumn{2}{c}{64.65$\pm$0.79}\\
&StableGL&\multicolumn{2}{c}{39.99$\pm$2.97}&\multicolumn{2}{c}{37.44$\pm$5.21}&\multicolumn{2}{c}{63.06$\pm$4.10}\\
  &\model{}&\multicolumn{2}{c}{\textbf{55.92$\pm$5.20}}&\multicolumn{2}{c}{\textbf{44.96$\pm$1.78}}&\multicolumn{2}{c}{\textbf{67.02$\pm$1.12}}\\
  \bottomrule
\end{tabular}
\label{tab::results-synthetic}
\end{table}

For \texttt{facebook-100}, we report the F1-score of \model{} in comparison with the baselines on each of the test graphs in Table \ref{tab::results-facebook}. For each method, we use SGC, GCN, GAT, and H2GCN as baselines, and only report the results of the best-performed GNN selected with the validation set. We provide the results when using the first set of training graphs. Complete results are provided in Appendix \ref{append:complete-results}. We observe that \model{} significantly outperforms best baselines with an average improvement of 3.8\%, showing that \model{} enhances the generalizability of GNNs across different domains.

 For \texttt{OGB-elliptic},  we plot the distribution of the F1 score averaged on all of the test graph snapshots with SGC as the backbone in Figure~\ref{fig::elliptic}. The rectangles represent the standard deviation, and the error bars illustrate the range of data samples. The dashed line shows the mean F1 score of the ERM method. As shown, \model{} significantly outperforms all baselines with a higher mean F1 score while demonstrating more stable performance, indicated by a smaller standard deviation, showing its ability to handle potentially more complex temporal shifts.

We also note that unlike other datasets, where in most cases the baseline graph OOD methods {the baseline graph OOD methods either slightly improve or achieve comparable performance as ERM}, on \texttt{facebook-100} and \texttt{OGB-elliptic}, most of them degrade the performance significantly. The reasons behind this could be that the false assumptions made by those methods fail to apply to the distribution shifts in real-life scenarios and the over-confidence brought by their strategies could further degrade the generalizability of the model. 

\subsection{The impact of confounder effect (RQ3)}

We create three synthetic graphs, \texttt{h-feat}, \texttt{qtr-feat}, and \texttt{full-feat}, to simulate scenarios with and without confounder effects between node features and neighborhood representations.  
\texttt{h-feat} is created such that the node features and neighborhood representations are dependent and act as confounders to each other. In contrast, \texttt{full-feat} and \texttt{qtr-feat} are designed so that these elements are independent, with no confounder effects. Details for creating these datasets are provided in Appendix \ref{append::synthetic-datasets}. 

We compare our model against baselines on the three synthetic datasets. We report the results of the best-performed one between SGC, GCN, GAT selected with the validation set. We also incorporate H2GCN to address potential heterophily in these datasets. Besides, H2GCN separately models node features and neighborhood representations, making its performance on $ERM$ a useful benchmark for our decoupling framework. 
We present the average Macro-F1 scores in Table \ref{tab::results-synthetic}. {We report the results for the best-performed one among SGC, GCN, and GAT due to page limitation, please refer to Table \ref{tab::appendix-results-synthetic} in Appendix for the full results.} For \texttt{qtr-feat} and \texttt{full-feat}, where node features and neighborhood representations are independent, employing H2GCN on ERM significantly improves the Macro-F1 score compared to other GNN models. However, in \texttt{h-feat}, where node features and neighborhood representations are correlated, the gains are modest and are outperformed by our model with any GNN backbone. This suggests that while the separate modeling approach of H2GCN can mitigate distribution shifts, it falls short in addressing confounder effects, thus underperforming when node features and neighborhood representations are correlated.

\section{Conclusion}
In this paper, we introduce a causal decoupling framework to improve out-of-distribution generalization on node classification.  We develop a comprehensive theoretical foundation to show that, by independently estimating the treatment effects of node features and neighborhood representations, the casual decoupling framework is robust when dealing with different types of distribution shifts. Following this theoretical basis, we also propose an implementable end-to-end framework, \model{}, that leverages casual inference technologies to estimate the decoupled representations and make final predictions. We demonstrate the effectiveness and the power of \model{} for node classification on both real-world datasets and synthetic datasets under different types of distribution shifts.

\bibliographystyle{unsrt}  
\bibliography{main}  

\begin{thebibliography}{10}

\bibitem{Wu2021_GNN_survey}
Zonghan Wu, Shirui Pan, Fengwen Chen, Guodong Long, Chengqi Zhang, and Philip~S. Yu.
\newblock A comprehensive survey on graph neural networks.
\newblock {\em IEEE Transactions on Neural Networks and Learning Systems}, 32(1):4–24, January 2021.

\bibitem{Liu2023FLOOD}
Yang Liu, Xiang Ao, Fuli Feng, Yunshan Ma, Kuan Li, Tat-Seng Chua, and Qing He.
\newblock Flood: A flexible invariant learning framework for out-of-distribution generalization on graphs.
\newblock In {\em KDD}, KDD '23, page 1548–1558, New York, NY, USA, 2023. Association for Computing Machinery.

\bibitem{arjovsky2020invariant}
Martin Arjovsky, Léon Bottou, Ishaan Gulrajani, and David Lopez-Paz.
\newblock Invariant risk minimization, 2020.

\bibitem{bai2020decaug}
Haoyue Bai, Rui Sun, Lanqing Hong, Fengwei Zhou, Nanyang Ye, Han-Jia Ye, S.-H.~Gary Chan, and Zhenguo Li.
\newblock Decaug: Out-of-distribution generalization via decomposed feature representation and semantic augmentation.
\newblock {\em Proceedings of the AAAI Conference on Artificial Intelligence}, 35(8):6705--6713, May 2021.

\bibitem{Krueger2020}
David Krueger, Ethan Caballero, J{\"{o}}rn{-}Henrik Jacobsen, Amy Zhang, Jonathan Binas, R{\'{e}}mi~Le Priol, and Aaron~C. Courville.
\newblock Out-of-distribution generalization via risk extrapolation (rex).
\newblock {\em CoRR}, abs/2003.00688, 2020.

\bibitem{Sagawa2019}
Shiori Sagawa, Pang~Wei Koh, Tatsunori~B. Hashimoto, and Percy Liang.
\newblock Distributionally robust neural networks for group shifts: On the importance of regularization for worst-case generalization.
\newblock {\em CoRR}, abs/1911.08731, 2019.

\bibitem{Shen2021}
Zheyan Shen, Jiashuo Liu, Yue He, Xingxuan Zhang, Renzhe Xu, Han Yu, and Peng Cui.
\newblock Towards out-of-distribution generalization: {A} survey.
\newblock {\em CoRR}, abs/2108.13624, 2021.

\bibitem{Ye2021theoretical}
Haotian Ye, Chuanlong Xie, Tianle Cai, Ruichen Li, Zhenguo Li, and Liwei Wang.
\newblock Towards a theoretical framework of out-of-distribution generalization, 2021.

\bibitem{gui2022good}
Shurui Gui, Xiner Li, Limei Wang, and Shuiwang Ji.
\newblock Good: A graph out-of-distribution benchmark, 2022.

\bibitem{wu2019-SGC}
Felix Wu, Amauri Souza, Tianyi Zhang, Christopher Fifty, Tao Yu, and Kilian Weinberger.
\newblock Simplifying graph convolutional networks.
\newblock In {\em Proceedings of the 36th International Conference on Machine Learning}, volume~97 of {\em Proceedings of Machine Learning Research}, pages 6861--6871. PMLR, 09--15 Jun 2019.

\bibitem{Robinson1988}
P.~M. Robinson.
\newblock Root-n-consistent semiparametric regression.
\newblock {\em Econometrica}, 56(4):931--954, 1988.

\bibitem{kaddour2021causal}
Jean Kaddour, Yuchen Zhu, Qi~Liu, Matt~J Kusner, and Ricardo Silva.
\newblock Causal effect inference for structured treatments.
\newblock In M.~Ranzato, A.~Beygelzimer, Y.~Dauphin, P.S. Liang, and J.~Wortman Vaughan, editors, {\em Advances in Neural Information Processing Systems}, volume~34, pages 24841--24854. Curran Associates, Inc., 2021.

\bibitem{Wu2022Handling}
Qitian Wu, Hengrui Zhang, Junchi Yan, and David Wipf.
\newblock Handling distribution shifts on graphs: An invariance perspective.
\newblock In {\em International Conference on Learning Representations}, 2022.

\bibitem{xia2023learning}
Donglin Xia, Xiao Wang, Nian Liu, and Chuan Shi.
\newblock Learning invariant representations of graph neural networks via cluster generalization.
\newblock In {\em Thirty-seventh Conference on Neural Information Processing Systems}, 2023.

\bibitem{zhang2023StableGL}
Shengyu Zhang, Yunze Tong, Kun Kuang, Fuli Feng, Jiezhong Qiu, Jin Yu, Zhou Zhao, Hongxia Yang, Zhongfei Zhang, and Fei Wu.
\newblock Stable prediction on graphs with agnostic distribution shifts.
\newblock In {\em Proceedings of The KDD'23 Workshop on Causal Discovery, Prediction and Decision}, volume 218 of {\em Proceedings of Machine Learning Research}, pages 49--74. PMLR, 07 Aug 2023.

\bibitem{zhu2021shiftrobust}
Qi~Zhu, Natalia Ponomareva, Jiawei Han, and Bryan Perozzi.
\newblock {Shift-robust GNNs}: Overcoming the limitations of localized graph training data.
\newblock {\em Advances in Neural Information Processing Systems}, 34, 2021.

\bibitem{wu2023energybased}
Qitian Wu, Yiting Chen, Chenxiao Yang, and Junchi Yan.
\newblock Energy-based out-of-distribution detection for graph neural networks.
\newblock In {\em The Eleventh International Conference on Learning Representations}, 2023.

\bibitem{platonov2024critical}
Oleg Platonov, Denis Kuznedelev, Michael Diskin, Artem Babenko, and Liudmila Prokhorenkova.
\newblock A critical look at the evaluation of gnns under heterophily: Are we really making progress?, 2024.

\bibitem{Zhu2020}
Jiong Zhu, Yujun Yan, Lingxiao Zhao, Mark Heimann, Leman Akoglu, and Danai Koutra.
\newblock Generalizing graph neural networks beyond homophily.
\newblock {\em CoRR}, abs/2006.11468, 2020.

\bibitem{fb100}
Amanda~L. Traud, Peter~J. Mucha, and Mason~A. Porter.
\newblock Social structure of facebook networks.
\newblock {\em Physica A: Statistical Mechanics and its Applications}, 391(16):4165–4180, August 2012.

\bibitem{elliplic}
Benedek Rozemberczki, Carl Allen, and Rik Sarkar.
\newblock {Multi-Scale attributed node embedding}.
\newblock {\em Journal of Complex Networks}, 9(2):cnab014, 05 2021.

\bibitem{Muandet2013domain}
Krikamol Muandet, David Balduzzi, and Bernhard Schölkopf.
\newblock Domain generalization via invariant feature representation, 2013.

\bibitem{Ahuja2020}
Kartik Ahuja, Karthikeyan Shanmugam, Kush~R. Varshney, and Amit Dhurandhar.
\newblock Invariant risk minimization games.
\newblock {\em CoRR}, abs/2002.04692, 2020.

\bibitem{veličković2018GAT}
Petar Veličković, Guillem Cucurull, Arantxa Casanova, Adriana Romero, Pietro Liò, and Yoshua Bengio.
\newblock Graph attention networks.
\newblock In {\em International Conference on Learning Representations}, 2018.

\bibitem{Kipf2016}
Thomas~N. Kipf and Max Welling.
\newblock Semi-supervised classification with graph convolutional networks.
\newblock {\em CoRR}, abs/1609.02907, 2016.

\bibitem{Li2021OOD-GNN}
Haoyang Li, Xin Wang, Ziwei Zhang, and Wenwu Zhu.
\newblock Ood-gnn: Out-of-distribution generalized graph neural network.
\newblock {\em IEEE Transactions on Knowledge and Data Engineering}, 35(7):7328--7340, 2023.

\bibitem{fan2021generalizing}
Shaohua Fan, Xiao Wang, Chuan Shi, Peng Cui, and Bai Wang.
\newblock Generalizing graph neural networks on out-of-distribution graphs.
\newblock {\em IEEE Transactions on pattern analysis and machine intelligence}, 46, 2024.

\bibitem{Sui2021}
Yongduo Sui, Xiang Wang, Jiancan Wu, Xiangnan He, and Tat{-}Seng Chua.
\newblock Deconfounded training for graph neural networks.
\newblock {\em CoRR}, abs/2112.15089, 2021.

\bibitem{wu2022deconfounding}
Ying-Xin Wu, Xiang Wang, An~Zhang, Xia Hu, Fuli Feng, Xiangnan He, and Tat-Seng Chua.
\newblock Deconfounding to explanation evaluation in graph neural networks, 2022.

\bibitem{chen2022learning}
Yongqiang Chen, Yonggang Zhang, Yatao Bian, Han Yang, MA~KAILI, Binghui Xie, Tongliang Liu, Bo~Han, and James Cheng.
\newblock Learning causally invariant representations for out-of-distribution generalization on graphs.
\newblock In {\em Advances in Neural Information Processing Systems}, 2022.

\bibitem{Li2022}
Haoyang Li, Ziwei Zhang, Xin Wang, and Wenwu Zhu.
\newblock Learning invariant graph representations for out-of-distribution generalization.
\newblock In {\em Advances in Neural Information Processing Systems}, volume~35, pages 11828--11841. Curran Associates, Inc., 2022.

\bibitem{Wu2022rationale}
Ying{-}Xin Wu, Xiang Wang, An~Zhang, Xiangnan He, and Tat{-}Seng Chua.
\newblock Discovering invariant rationales for graph neural networks.
\newblock {\em CoRR}, abs/2201.12872, 2022.

\bibitem{buffelli2022sizeshiftreg}
Davide Buffelli, Pietro Lio, and Fabio Vandin.
\newblock Sizeshiftreg: a regularization method for improving size-generalization in graph neural networks.
\newblock In {\em Proceedings of the 36th International Conference on Neural Information Processing Systems}, 2022.

\bibitem{Fan2022debiased}
Shaohua Fan, Xiao Wang, Chuan Shi, Kun Kuang, Nian Liu, and Bai Wang.
\newblock Debiased graph neural networks with agnostic label selection bias.
\newblock {\em CoRR}, abs/2201.07708, 2022.

\bibitem{li2022survey}
Haoyang Li, Xin Wang, Ziwei Zhang, and Wenwu Zhu.
\newblock Out-of-distribution generalization on graphs: A survey, 2022.

\bibitem{Kuang2018}
Kun Kuang, Ruoxuan Xiong, Peng Cui, Susan Athey, and Bo~Li.
\newblock Stable prediction across unknown environments.
\newblock {\em CoRR}, abs/1806.06270, 2018.

\bibitem{Ma2019}
Jianxin Ma, Peng Cui, Kun Kuang, Xin Wang, and Wenwu Zhu.
\newblock Disentangled graph convolutional networks.
\newblock In Kamalika Chaudhuri and Ruslan Salakhutdinov, editors, {\em Proceedings of the 36th International Conference on Machine Learning}, volume~97 of {\em Proceedings of Machine Learning Research}, pages 4212--4221. PMLR, 09--15 Jun 2019.

\bibitem{Liu2020}
Yanbei Liu, Xiao Wang, Shu Wu, and Zhitao Xiao.
\newblock Independence promoted graph disentangled networks.
\newblock {\em Proceedings of the AAAI Conference on Artificial Intelligence}, 34(04):4916--4923, Apr. 2020.

\bibitem{Heckman2018}
James~J. Heckman, John~Eric Humphries, and Gregory Veramendi.
\newblock Returns to education: The causal effects of education on earnings, health, and smoking.
\newblock {\em Journal of Political Economy}, 126(S1):S197--S246, 2018.

\bibitem{Prosperi2020CausalIA}
Mattia C.~F. Prosperi, Yi~Guo, Matt Sperrin, James~S. Koopman, Jae Min, Xing He, Shannan~N. Rich, Mo~Wang, Iain~E. Buchan, and Jiang Bian.
\newblock Causal inference and counterfactual prediction in machine learning for actionable healthcare.
\newblock {\em Nature Machine Intelligence}, 2:369 -- 375, 2020.

\bibitem{Yao2018}
Liuyi Yao, Sheng Li, Yaliang Li, Mengdi Huai, Jing Gao, and Aidong Zhang.
\newblock Representation learning for treatment effect estimation from observational data.
\newblock In {\em Advances in Neural Information Processing Systems}, volume~31. Curran Associates, Inc., 2018.

\bibitem{chernozhukov2017doubledebiased}
Victor Chernozhukov, Denis Chetverikov, Mert Demirer, Esther Duflo, Christian Hansen, Whitney Newey, and James Robins.
\newblock Double/debiased machine learning for treatment and causal parameters, 2017.

\bibitem{Abadie2006}
Alberto Abadie and Guido~W. Imbens.
\newblock Large sample properties of matching estimators for average treatment effects.
\newblock {\em Econometrica}, 74(1):235--267, 2006.

\bibitem{Austin2011}
Peter~C. Austin.
\newblock An introduction to propensity score methods for reducing the effects of confounding in observational studies.
\newblock {\em Multivariate Behavioral Research}, 46(3):399--424, 2011.
\newblock PMID: 21818162.

\bibitem{Rubin1978}
Donald~B. Rubin.
\newblock Using multivariate matched sampling and regression adjustment to control bias in observational studies.
\newblock {\em ETS Research Bulletin Series}, 1978(2):i--33, 1978.

\bibitem{Chang2017}
Yale Chang and Jennifer Dy.
\newblock Informative subspace learning for counterfactual inference.
\newblock {\em Proceedings of the AAAI Conference on Artificial Intelligence}, 31(1), Feb. 2017.

\bibitem{Rubin1973}
Donald~B. Rubin.
\newblock Matching to remove bias in observational studies.
\newblock {\em Biometrics}, 29(1):159--183, 1973.

\bibitem{Hill2011}
Jennifer~L. Hill.
\newblock Bayesian nonparametric modeling for causal inference.
\newblock {\em Journal of Computational and Graphical Statistics}, 20(1):217--240, 2011.

\bibitem{wager2017estimation}
Stefan Wager and Susan Athey.
\newblock Estimation and inference of heterogeneous treatment effects using random forests.
\newblock {\em Journal of the American Statistical Association}, 113(523):1228--1242, 2018.

\bibitem{soren2019}
Sören~R. Künzel, Jasjeet~S. Sekhon, Peter~J. Bickel, and Bin Yu.
\newblock Metalearners for estimating heterogeneous treatment effects using machine learning.
\newblock {\em Proceedings of the National Academy of Sciences}, 116(10):4156--4165, 2019.

\bibitem{Funk2011}
Michele~Jonsson Funk, Daniel Westreich, Chris Wiesen, Til Stürmer, M.~Alan Brookhart, and Marie Davidian.
\newblock {Doubly Robust Estimation of Causal Effects}.
\newblock {\em American Journal of Epidemiology}, 173(7):761--767, 03 2011.

\bibitem{shalit2017estimating}
Uri Shalit, Fredrik~D. Johansson, and David Sontag.
\newblock Estimating individual treatment effect: generalization bounds and algorithms, 2017.

\bibitem{johansson2018learning}
Fredrik Johansson, Uri Shalit, and David Sontag.
\newblock Learning representations for counterfactual inference.
\newblock In {\em Proceedings of The 33rd International Conference on Machine Learning}, volume~48 of {\em Proceedings of Machine Learning Research}, pages 3020--3029, New York, New York, USA, 20--22 Jun 2016. PMLR.

\bibitem{shi2019adapting}
Claudia Shi, David~M. Blei, and Victor Veitch.
\newblock Adapting neural networks for the estimation of treatment effects, 2019.

\bibitem{athey2019estimating}
Susan Athey and Stefan Wager.
\newblock Estimating treatment effects with causal forests: An application.
\newblock {\em Observational Studies}, 5, 2019.

\bibitem{Bica2020}
Ioana Bica, James Jordon, and Mihaela van~der Schaar.
\newblock Estimating the effects of continuous-valued interventions using generative adversarial networks.
\newblock {\em CoRR}, abs/2002.12326, 2020.

\bibitem{kallus2018deepmatch}
Nathan Kallus.
\newblock {D}eep{M}atch: Balancing deep covariate representations for causal inference using adversarial training.
\newblock In Hal~Daumé III and Aarti Singh, editors, {\em Proceedings of the 37th International Conference on Machine Learning}, volume 119 of {\em Proceedings of Machine Learning Research}, pages 5067--5077. PMLR, 13--18 Jul 2020.

\bibitem{Shonosuke2020}
Shonosuke Harada and Hisashi Kashima.
\newblock Graphite: Estimating individual effects of graph-structured treatments.
\newblock {\em CoRR}, abs/2009.14061, 2020.

\bibitem{cora}
Andrew McCallum, Kamal Nigam, Jason D.~M. Rennie, and Kristie Seymore.
\newblock Automating the construction of internet portals with machine learning.
\newblock {\em Information Retrieval}, 3(2):127--163, 2000.

\bibitem{citeseer}
C.~Lee Giles, Kurt~D. Bollacker, and Steve Lawrence.
\newblock Citeseer: an automatic citation indexing system.
\newblock In {\em Proceedings of the Third ACM Conference on Digital Libraries}, DL '98, page 89–98, New York, NY, USA, 1998. Association for Computing Machinery.

\bibitem{coauthor-cs}
Oleksandr Shchur, Maximilian Mumme, Aleksandar Bojchevski, and Stephan Günnemann.
\newblock Pitfalls of graph neural network evaluation, 2019.

\bibitem{amazon-photo}
Julian McAuley, Christopher Targett, Qinfeng Shi, and Anton van~den Hengel.
\newblock Image-based recommendations on styles and substitutes, 2015.

\end{thebibliography}
\newpage
\section*{Appendix}
\setcounter{section}{0}
\section{Related work}
\subsection{OOD Generalization on Graphs}
\label{related-work-ood}

Abundant pioneer works~\cite{Muandet2013domain, Ahuja2020, bai2020decaug, Krueger2020, Sagawa2019, Shen2021, Ye2021theoretical} have been dedicated to addressing the out-of-distribution (OOD) generalization problem. The primary objective of OOD generalization is to train models within a specific domain and anticipate their robust generalization to test domains from potentially distinct distributions~\cite{Ye2021theoretical}. While most of these methods are tailored to deal with distribution shifts on tabular data or images, their performance is restricted when confronted with more complex data structures. With the success of Graph Neural Networks (GNNs)~\cite{veličković2018GAT, Kipf2016}, more recent research starts to address the OOD problem on graphs~\cite{Li2021OOD-GNN, fan2021generalizing, Sui2021, wu2022deconfounding, chen2022learning, Li2022, Wu2022rationale, buffelli2022sizeshiftreg, Fan2022debiased, Wu2022Handling, Liu2023FLOOD, xia2023learning, zhang2023StableGL, zhu2021shiftrobust}. OOD generalization on graphs is more intractable due to the \textit{non-euclidean} nature of graph data structures and the subtlety of different types of distribution shifts~\cite{li2022survey}. 

One consensus of OOD generalization is that, when domain knowledge is unavailable, knowledge transfer to a new domain is impossible without structural assumption on data generation processes ~\cite{Wu2022Handling}. A general assumption made by most OOD generalization methods is that there exist ``true'' correlations between input data features and their labels. These correlations remain invariant across different domains, and they aim to identify these true connections while removing the spurious ones~\cite{li2022survey}. Most of the existing graph OOD generalization methods leverage this assumption in different manners or extend it to more specific forms. 
Approaches~\cite{Li2021OOD-GNN, fan2021generalizing, Fan2022debiased} backed by \textit{confounder balancing}~\cite{Kuang2018}
from casual theories remove the correlations between casual and non-causal (spurious) aspects (in the forms of representations~\cite{Li2021OOD-GNN}, subgraphs~\cite{fan2021generalizing}, or samples~\cite{Fan2022debiased}, etc.) such that the model can focus on the casual ones. 
Some methods~\cite{Li2022, Wu2022rationale, Wu2022Handling, Liu2023FLOOD, xia2023learning, zhang2023StableGL} leverage \textit{invariance principle}~\cite{arjovsky2020invariant} from causality and assumes that there exists a portion of information in the input (e.g., subgraphs) that is invariant to label predictions across different environments. This approach often requires access to multiple environments/domains during the training process, which are not always available. 
Structural causal graphs (SCGs) have also been used for assumptions in data generation processes~\cite{Sui2021, wu2022deconfounding, chen2022learning}. 
They focus on making unbiased estimations of causal relationships via \textit{do}-calculus~\cite{Sui2021} or back-door adjustments~\cite{wu2022deconfounding}.  
Zhu \textit{et al.}~\cite{zhu2021shiftrobust} assume that training data are from a biased data generation process while test data are unbiased. 
Buffelli \textit{et al.}~\cite{buffelli2022sizeshiftreg} improves model generalization on smaller or larger test graphs by minimizing the discrepancy between the learned representation on the original graph and the coarsened graph.

Among the above-mentioned methods, most of them~\cite{Li2021OOD-GNN, fan2021generalizing, Sui2021, wu2022deconfounding, chen2022learning, Li2022, Wu2022rationale, buffelli2022sizeshiftreg} are tailored for graph-level or link-level tasks, and it is non-trivial to adapt them to node-level tasks. 
For the rest of the methods~\cite{Fan2022debiased, Wu2022Handling, Liu2023FLOOD, xia2023learning, zhang2023StableGL, zhu2021shiftrobust} that can be applied on node-level tasks, an ego-net is often the starting point for any further actions to be taken (to generate an invariant subgraph, or to learn the invariant representation of it, etc). 
Current approaches involve aggregating a central node and its neighborhood into a unified representation (e.g., GCNs~\cite{Kipf2016} or GATs~\cite{veličković2018GAT}), but they tend to overlook and filter out potential complex dependencies (or independencies) between them.
Certain GNNs such as H2GCN~\cite{Zhu2020} can model the central node and its neighbors' representations separately through the concatenation operation during the aggregation, but inferences on their relations are still missing.  Disentangled graph learning \cite{Ma2019, Liu2020} investigates latent factors that may cause the formation of an edge between a node and its neighbors. This method can also improve the OOD generalizability of GNNs, but it focuses on explaining the existence of edges with node representations, which is fundamentally different from our focus. 

\subsection{Causal Effect Estimation}
Estimating the causal effect of treatment plays a crucial role in many domains~\cite {Heckman2018, Prosperi2020CausalIA}. With the presence of the confounder, two types of studies, dubbed \textit{randomized controlled trials (RCTs)} and\textit{ observational studies}, are often conducted to achieve an unbiased estimation of casual effects~\cite{Yao2018}. 
Often, RCTs are expensive, unethical, or infeasible~\cite{Yao2018}, leaving observational studies the only option. 
The challenges of observational studies are that the counterfactual samples are missing from observational data, and the estimation of counterfactual output is often biased due to the treatment selection bias caused by the confounder effect~\cite{chernozhukov2017doubledebiased}. 
Traditional solutions include matching methods~\cite{Abadie2006, Austin2011, Rubin1978, Chang2017, Rubin1973}, tree-based methods~\cite{Hill2011, wager2017estimation}, and regression-based methods~\cite{soren2019, Funk2011}. Recently, representation learning methods~\cite{shalit2017estimating, Yao2018, johansson2018learning, shi2019adapting, athey2019estimating, Bica2020, kallus2018deepmatch} have also been widely studied and demonstrate impressive performance. The above-mentioned studies focus on binary or categorical treatments, which are difficult to apply when the treatments are real-valued and structured. GraphITE~\cite{Shonosuke2020} is proposed to deal with graph-structured treatment, which mitigates observation biases by reinforcing the independence between the treatment and covariate representations in the notion of Hilbert-Schmidt Independence Criterion (HSIC).  Robinson decomposition~\cite{kaddour2021causal} is used to identify the distinct contribution of the treatment and the covariates and generalizes it such that the treatment can be vectorized to a continuous embedding. 
\section{Pseudo code}\label{append:pseudo}
We provide the pseudo-code for \model{} in Algorithm \ref{alg::model}. 
\begin{algorithm}
	\caption{The Proposed Method \model}
 \label{alg::model}
	\begin{algorithmic}[1]
     \State \textbf{Input}: node features $\mathbf{X}\in \mathbb{R}^{n\times d}$, unnormalized adjacency matrix $A\in\{0,1\}^{n\times n}$
      \State \textbf{Output}: predicted class labels $\mathbf{Y}\in{\mathbb{I}}^{n\times k}$ 
      \While {not converged} 
      \State Evaluate $J_{p, p'}(\rho, \rho')$ based on Equation \ref{J_pp}
      \State Update $\rho \leftarrow \rho - \hat {\Delta}_{\rho}J_{p, p'}(\rho, \rho')$
        \State Update $\rho' \leftarrow \rho' - \hat {\Delta}_{\rho'}J_{p, p'}(\rho, \rho')$
      \EndWhile
   \State Evaluate $\mathbf{h}_i^{\mathbf{a}}$ and $\mathbf{m}_i^{\mathbf{a}}$
   \While {not converged} 
   \State Sample mini-batch $\{ \mathbf{x}_i, \mathbf{h}_i^{\mathbf{a}}\}^b_{i=1}$
   \State Evaluate $J_{g}(\psi^A)$
   \State Update $\theta^A \leftarrow \theta^A - \hat {\Delta}_{\theta^A}J_{m}(\theta^A)$
   \EndWhile
   
   \While {not converged} 
   \State Sample mini-batch $\{ \mathbf{x}_i, \mathbf{h}_i^{\mathbf{a}}\}^b_{i=1}$
   \State Evaluate $J_{m}(\theta^A)$, $J_{e}(\eta^X)$
   \For {step = 1 \textbf{to} \textit{step\_size}}
   \State Update $\theta^A \leftarrow \theta^A - \hat {\Delta}_{\theta^A}
   J_{m}(\theta^A)$
   \EndFor
   \State Update $\eta^A \leftarrow \eta^A - \hat {\Delta}_{\eta^A}
   J_{e}(\eta^X)$
   \EndWhile

   \While {not converged} 
   \State Sample mini-batch $\{ \mathbf{x}_i, \mathbf{h}_i^{\mathbf{a}}, \mathbf{m}_i^{\mathbf{a}}\}^b_{i=1}$
   \State Evaluate $J_{h}(\phi^X)$, $J_{e}(\eta^X)$
   \For {step = 1 \textbf{to} \textit{step\_size}}
   \State Update $\phi^X \leftarrow \phi^X - \hat {\Delta}_{\phi^X}
   J_{m}(\phi^X)$
   \EndFor
   \State Update $\eta^X \leftarrow \eta^X - \hat {\Delta}_{\eta^X}
   J_{e}(\eta^X)$
   \EndWhile

  \State Evaluate $\mathbf{E}[\mathbf{y}|C=\mathbf{a}, do(T=\mathbf{x}')]$, $\mathbf{E}[\mathbf{y}|C=\mathbf{x}, do(T=\mathbf{a}')]$ based on Equations in Section~\ref{sec:counterfact}
   \State Evaluate ${\Psi_{\mathbf{a}}}(\mathbf{a})$, ${\Psi_{\mathbf{x}}}(\mathbf{x})$ based on Equation \ref{eq:tau_a} and \ref{eq:tau_x}
  \State Evaluate $\mathbf{y}$ based on Equation \ref{eq:final_prediction}  
	\end{algorithmic} 
\end{algorithm}

\section{Important mathematical notations}
We provide a summary of important mathematical notations in Table \ref{tab::notation}.
\begin{table}
\caption{Important notations and descriptions.}
\label{sum-data}
\begin{center}
\begin{tabular}{ll}
\toprule
\textbf{Notation}  & \textbf{Description}\\
\midrule
$n, d, k, o$ & $\#$ nodes, feature size, $\#$ classes, hidden size\\
$\mathbf{x}$, $\mathbf{x}_i$, $\mathbf{X}\in\mathbb{R}^{n\times d}$ & node feature vector, node feature of $i$-th instance, node feature matrix\\
$\mathbf{y}$, $\mathbf{y}_i$, 
$\mathbf{Y}\in\mathbb{I}^{n\times k}$ & node label vector, label of $i$-th instance, node label matrix\\
$\mathbf{a}$, $\mathbf{a}_i$, $\mathbf{A}\in\mathbb{R}^{n\times o}$ & neighborhood representation, neighborhood representation, of $i$-th instance, \\&neighborhood representation matrix\\
$A\in\mathbb{I}^{n\times n}$ & unnormalized adjacency matrix\\
$p(\cdot), p'(\cdot)$ & GNN embedding layer, MLP layer\\
$\rho, \rho'$ & parameters of  $p(\cdot), p'(\cdot)$\\
$\mathbf{h}_i^{\mathbf{a}}$,  $\mathbf{m}_i^{\mathbf{a}}$ & hidden neighborhood representation of node $i$,  and its project to output space\\
$g^{X}(\cdot), g^{A}(\cdot)$ & confounder representation function for SCM-F and SCM-A\\
${\psi^X}, {\psi^A}$ & parameters for $g^{X}(\cdot), g^{A}(\cdot)$\\
$h^{X}(\cdot), h^{A}(\cdot)$ & treatment representation function for SCM-F and SCM-A\\
${\phi^X}, {\phi^A}$ & parameters for $h^{X}(\cdot), h^{A}(\cdot)$\\
$m^{X}(\cdot), m^{A}(\cdot)$ & confounder predicting function for SCM-F and SCM-A\\
${\theta^X}, {\theta^A}$ & parameters for $m^{X}(\cdot), m^{A}(\cdot)$\\
$e^{X}(\cdot), e^{A}(\cdot)$ & propensity feature function for SCM-F and SCM-A\\
${\eta^X}, {\eta^A}$ & parameters for $e^{X}(\cdot), e^{A}(\cdot)$\\
$\mathbf{x}'$, $\mathbf{a}'$ & counterfactual node feature and neighborhood representation\\

\bottomrule
\end{tabular}
\end{center}
\label{tab::notation}
\end{table} 


\section{Discussion on spillover effect}
\label{append:spillover}
 \begin{assumption}
 \label{assumption_LLN}
     The Law of Large Numbers is invoked, ensuring that the sample size is sufficiently large such that the observed distribution of the random variable z remains stable and converges towards its theoretical distribution. 
 \end{assumption}
\begin{observation}
Given $\{\mathbf{z}_j\}^{n}_{j=1}$ are sampled from a specific distribution. Under assumption \ref{assumption_1}, the expectation of the neighborhood representation $\mathbf{a}_i$ of node $i$ is dependent on $\mathbf{z}_i$ and the distribution of $z$. Under assumption \ref{assumption_LLN}, the change of other individual nodes has a negligible spillover effect on the expectation of $\mathbf{a}_i$.
\end{observation}

\section{Analysis on distribution shifts}
\label{appendix-distribution-shifts}
\begin{proposition}
Under covariate shift, the correlation between $\mathbf{x}$ and $\mathbf{a}$ shifts (\textit{e.g.} $P^\text{train}(\mathbf{a}|\mathbf{x}) \neq P^\text{test}(\mathbf{a}|\mathbf{x})$ or $P^\text{train}(\mathbf{x}|\mathbf{a}) \neq P^\text{test}(\mathbf{x}|\mathbf{a})$). Consequently, the conditional distribution of $\mathbf{y}$ given both $\mathbf{x}$ and $\mathbf{a}$ shifts (\textit{e.g.} $P^\text{train}(\mathbf{y}|\mathbf{x},\mathbf{a}) \neq P^\text{test}(\mathbf{y}|\mathbf{x},\mathbf{a})$). However, the conditional distribution of $\mathbf{y}$ given $\mathbf{x}$ alone remain invariant (\textit{e.g.} $P^\text{train}(\mathbf{y}|\mathbf{x}) = P^\text{test}(\mathbf{y}|\mathbf{x})$). 
\end{proposition}

\begin{proof}
Under covariate shift, we have:
\[
P^\text{train}(\mathbf{z}) \neq P^\text{test}(\mathbf{z})
\]
while the transformation matrices and bias vectors for generating \(\mathbf{x}\), \(\mathbf{y}\), and \(\mathbf{a}\) remain the same between training and testing datasets:
\[
\mathcal{M}_{x}^\text{train} = \mathcal{M}_{x}^\text{test}, \quad \mathbf{b}_{x}^\text{train} = \mathbf{b}_{x}^\text{test}
\]
\[
\mathcal{M}_{y}^\text{train} = \mathcal{M}_{y}^\text{test}, \quad \mathbf{b}_{y}^\text{train} = \mathbf{b}_{y}^\text{test}
\]
\[
\mathcal{M}_{s}^\text{train} = \mathcal{M}_{s}^\text{test}, \quad \mathcal{M}_{o}^\text{train} = \mathcal{M}_{o}^\text{test}
\]

The generation processes are:
\[
\mathbf{x}_i = \mathcal{M}_{x} \mathbf{z}_i + \mathbf{b}_{x}
\]
\[
\mathbf{y}_i = \mathcal{M}_{y} \mathbf{z}_i + \mathbf{b}_{y}
\]
\[
p(A_{ij} = 1) = c \left( \| \mathcal{M}_{s} \mathbf{z}_i - \mathcal{M}_{o} \mathbf{z}_j \|^2_2 + 1 \right)^{-1}
\]
where \(\mathbf{a}_i\) is defined as:
\[
\mathbf{a}_i = \text{agg}(\{ \mathcal{M}_a \mathbf{z}_j \}) \quad \text{where} \quad j \in \mathcal{N}_i^1 \cup \mathcal{N}_i^2 \cup \ldots \cup \mathcal{N}_i^l
\]

Given that \(P(\mathbf{z})\) changes, the marginal distribution of \(\mathbf{x}\) changes accordingly:
\[
P^\text{train}(\mathbf{x}) \neq P^\text{test}(\mathbf{x})
\]

Since \(\mathbf{a}_i\) is an aggregation of transformed neighbors' latent vectors \(\mathbf{z}_j\), and \(P(\mathbf{z})\) changes, the distribution of \(\mathbf{a}_i\) also changes:
\[
P^\text{train}(\mathbf{a}) \neq P^\text{test}(\mathbf{a})
\]

To show that the joint distribution \(P(\mathbf{x}, \mathbf{a})\) changes, consider that \(\mathbf{a}\) is derived from \(\mathbf{z}\) through a different transformation \(\mathcal{M}_a\) and aggregation function. Since \(\mathbf{x}\) and \(\mathbf{a}\) are both functions of \(\mathbf{z}\), any change in the distribution of \(\mathbf{z}\) will induce changes in both \(P(\mathbf{x})\) and \(P(\mathbf{a})\). Given that \(P(\mathbf{x})\) and \(P(\mathbf{a})\) change independently due to the change in \(P(\mathbf{z})\), the joint distribution \(P(\mathbf{x}, \mathbf{a})\) must also change because the dependencies between \(\mathbf{x}\) and \(\mathbf{a}\) are functions of \(\mathbf{z}\):
\[
P^\text{train}(\mathbf{x}, \mathbf{a}) \neq P^\text{test}(\mathbf{x}, \mathbf{a})
\]

This implies that:
\[
P^\text{train}(\mathbf{a}|\mathbf{x}) = \frac{P^\text{train}(\mathbf{x}, \mathbf{a})}{P^\text{train}(\mathbf{x})} \neq \frac{P^\text{test}(\mathbf{x}, \mathbf{a})}{P^\text{test}(\mathbf{x})} = P^\text{test}(\mathbf{a}|\mathbf{x})
\]
\[
P^\text{train}(\mathbf{x}|\mathbf{a}) = \frac{P^\text{train}(\mathbf{x}, \mathbf{a})}{P^\text{train}(\mathbf{a})} \neq \frac{P^\text{test}(\mathbf{x}, \mathbf{a})}{P^\text{test}(\mathbf{a})} = P^\text{test}(\mathbf{x}|\mathbf{a})
\]

The generation process for \(\mathbf{y}\) is:
\[
\mathbf{y}_i = \mathcal{M}_{y} \mathbf{z}_i + \mathbf{b}_{y}
\]

Since \(\mathcal{M}_{y}\) and \(\mathbf{b}_{y}\) are the same across training and testing, and \(\mathbf{z}_i\) influences \(\mathbf{x}_i\) and \(\mathbf{a}_i\) through \(\mathcal{M}_{x}\), \(\mathcal{M}_{s}\), \(\mathcal{M}_{o}\), and \(\mathcal{M}_{a}\), the change in \(P(\mathbf{z})\) affects \(P(\mathbf{x})\) and \(P(\mathbf{a})\). Consequently, the joint distribution \(P(\mathbf{y}, \mathbf{x}, \mathbf{a})\) changes:
\[
P^\text{train}(\mathbf{y}, \mathbf{x}, \mathbf{a}) \neq P^\text{test}(\mathbf{y}, \mathbf{x}, \mathbf{a})
\]

Thus, the conditional distribution:
\[
P^\text{train}(\mathbf{y}|\mathbf{x}, \mathbf{a}) = \frac{P^\text{train}(\mathbf{y}, \mathbf{x}, \mathbf{a})}{P^\text{train}(\mathbf{x}, \mathbf{a})} \neq \frac{P^\text{test}(\mathbf{y}, \mathbf{x}, \mathbf{a})}{P^\text{test}(\mathbf{x}, \mathbf{a})} = P^\text{test}(\mathbf{y}|\mathbf{x}, \mathbf{a})
\]

Since \(\mathbf{y}_i = \mathcal{M}_{y} \mathbf{z}_i + \mathbf{b}_{y}\) and \(\mathcal{M}_{y}\) and \(\mathbf{b}_{y}\) are invariant, the conditional distribution \(P(\mathbf{y}|\mathbf{z})\) remains invariant:
\[
P^\text{train}(\mathbf{y}|\mathbf{z}) = P^\text{test}(\mathbf{y}|\mathbf{z})
\]

Given that \(\mathbf{x}_i = \mathcal{M}_{x} \mathbf{z}_i + \mathbf{b}_{x}\) and \(\mathcal{M}_{x}\) and \(\mathbf{b}_{x}\) are invariant, the conditional distribution \(P(\mathbf{y}|\mathbf{x})\) also remains invariant:
\[
P^\text{train}(\mathbf{y}|\mathbf{x}) = P^\text{test}(\mathbf{y}|\mathbf{x})
\]

\end{proof}

\begin{proposition}
Under concept shift-F, the dependencies between $\mathbf{x}$ and $\mathbf{a}$ shifts (\textit{e.g.} $P^\text{train}(\mathbf{a}|\mathbf{x}) \neq P^\text{test}(\mathbf{a}|\mathbf{x})$ or $P^\text{train}(\mathbf{x}|\mathbf{a}) \neq P^\text{test}(\mathbf{x}|\mathbf{a})$), and the conditional distribution of $Y$ given $F$ shifts (\textit{e.g.} $P^\text{train}(\mathbf{y}|\mathbf{x}) \neq P^\text{test}(\mathbf{y}|\mathbf{x})$). Consequently, the conditional distribution of $\mathbf{y}$ given both $\mathbf{x}$ and $\mathbf{a}$ shifts (\textit{e.g.} $P^\text{train}(\mathbf{y}|\mathbf{x},\mathbf{a}) \neq P^\text{test}(\mathbf{y}|\mathbf{x},\mathbf{a})$). However, the conditional distribution of $Y$ given $A$ alone remain invariant (\textit{e.g.}  $P^\text{train}(\mathbf{y}|\mathbf{a}) = P^\text{test}(\mathbf{y}|\mathbf{a})$). 
\end{proposition}

\begin{proof}
Under concept shift-F, we have:
\[
P^\text{train}(\mathbf{z}) = P^\text{test}(\mathbf{z})
\]
while the transformation matrices and bias vectors for generating \(\mathbf{x}\) change, but those for \(\mathbf{a}\) and \(\mathbf{y}\) remain the same between training and testing datasets:
\[
\mathcal{M}_{x}^\text{train} \neq \mathcal{M}_{x}^\text{test}, \quad \mathbf{b}_{x}^\text{train} \neq \mathbf{b}_{x}^\text{test}
\]
\[
\mathcal{M}_{y}^\text{train} = \mathcal{M}_{y}^\text{test}, \quad \mathbf{b}_{y}^\text{train} = \mathbf{b}_{y}^\text{test}
\]
\[
\mathcal{M}_{s}^\text{train} = \mathcal{M}_{s}^\text{test}, \quad \mathcal{M}_{o}^\text{train} = \mathcal{M}_{o}^\text{test}
\]

The generation processes are:
\[
\mathbf{x}_i^\text{train} = \mathcal{M}_{x}^\text{train} \mathbf{z}_i + \mathbf{b}_{x}^\text{train}
\]
\[
\mathbf{x}_i^\text{test} = \mathcal{M}_{x}^\text{test} \mathbf{z}_i + \mathbf{b}_{x}^\text{test}
\]
\[
\mathbf{y}_i = \mathcal{M}_{y} \mathbf{z}_i + \mathbf{b}_{y}
\]
\[
p(A_{ij} = 1) = c \left( \| \mathcal{M}_{s} \mathbf{z}_i - \mathcal{M}_{o} \mathbf{z}_j \|^2_2 + 1 \right)^{-1}
\]
where \(\mathbf{a}_i\) is defined as:
\[
\mathbf{a}_i = \text{agg}(\{ \mathcal{M}_a \mathbf{z}_j \}) \quad \text{where} \quad j \in \mathcal{N}_i^1 \cup \mathcal{N}_i^2 \cup \ldots \cup \mathcal{N}_i^l
\]

Given that \(P(\mathbf{z})\) remains the same, the marginal distribution of \(\mathbf{a}\) does not change:
\[
P^\text{train}(\mathbf{a}) = P^\text{test}(\mathbf{a})
\]

However, the change in the transformation matrix \(\mathcal{M}_{x}\) and bias vector \(\mathbf{b}_{x}\) implies that the marginal distribution of \(\mathbf{x}\) changes:
\[
P^\text{train}(\mathbf{x}) \neq P^\text{test}(\mathbf{x})
\]

Since \(\mathbf{a}\) is an aggregation of transformed neighbors' latent vectors \(\mathbf{z}_j\), and the generation process for \(\mathbf{a}\) remains unchanged, the joint distribution \(P(\mathbf{x}, \mathbf{a})\) changes due to the change in \(P(\mathbf{x})\):
\[
P^\text{train}(\mathbf{x}, \mathbf{a}) \neq P^\text{test}(\mathbf{x}, \mathbf{a})
\]

This implies that:
\[
P^\text{train}(\mathbf{a}|\mathbf{x}) = \frac{P^\text{train}(\mathbf{x}, \mathbf{a})}{P^\text{train}(\mathbf{x})} \neq \frac{P^\text{test}(\mathbf{x}, \mathbf{a})}{P^\text{test}(\mathbf{x})} = P^\text{test}(\mathbf{a}|\mathbf{x})
\]
\[
P^\text{train}(\mathbf{x}|\mathbf{a}) = \frac{P^\text{train}(\mathbf{x}, \mathbf{a})}{P^\text{train}(\mathbf{a})} \neq \frac{P^\text{test}(\mathbf{x}, \mathbf{a})}{P^\text{test}(\mathbf{a})} = P^\text{test}(\mathbf{x}|\mathbf{a})
\]

The generation process for \(\mathbf{y}\) is:
\[
\mathbf{y}_i = \mathcal{M}_{y} \mathbf{z}_i + \mathbf{b}_{y}
\]

Since \(\mathcal{M}_{y}\) and \(\mathbf{b}_{y}\) are the same across training and testing, and \(\mathbf{z}_i\) influences \(\mathbf{x}_i\) through \(\mathcal{M}_{x}\), the change in \(\mathcal{M}_{x}\) affects the joint distribution \(P(\mathbf{y}, \mathbf{x})\):
\[
P^\text{train}(\mathbf{y}, \mathbf{x}) \neq P^\text{test}(\mathbf{y}, \mathbf{x})
\]

Thus, the conditional distribution:
\[
P^\text{train}(\mathbf{y}|\mathbf{x}) = \frac{P^\text{train}(\mathbf{y}, \mathbf{x})}{P^\text{train}(\mathbf{x})} \neq \frac{P^\text{test}(\mathbf{y}, \mathbf{x})}{P^\text{test}(\mathbf{x})} = P^\text{test}(\mathbf{y}|\mathbf{x})
\]

The change in the joint distribution \(P(\mathbf{x}, \mathbf{a})\) implies a change in the joint distribution \(P(\mathbf{y}, \mathbf{x}, \mathbf{a})\) since \(\mathbf{y}\) depends on both \(\mathbf{x}\) and \(\mathbf{a}\):
\[
P^\text{train}(\mathbf{y}, \mathbf{x}, \mathbf{a}) \neq P^\text{test}(\mathbf{y}, \mathbf{x}, \mathbf{a})
\]

Thus, the conditional distribution:
\[
P^\text{train}(\mathbf{y}|\mathbf{x}, \mathbf{a}) = \frac{P^\text{train}(\mathbf{y}, \mathbf{x}, \mathbf{a})}{P^\text{train}(\mathbf{x}, \mathbf{a})} \neq \frac{P^\text{test}(\mathbf{y}, \mathbf{x}, \mathbf{a})}{P^\text{test}(\mathbf{x}, \mathbf{a})} = P^\text{test}(\mathbf{y}|\mathbf{x}, \mathbf{a})
\]

Since \(\mathbf{y}\) and \(\mathbf{a}\) are both derived from \(\mathbf{z}\) through invariant transformation matrices and bias vectors, and \(P(\mathbf{z})\) remains the same, the joint distribution \(P(\mathbf{y}, \mathbf{a})\) remains unchanged:
\[
P^\text{train}(\mathbf{y}, \mathbf{a}) = P^\text{test}(\mathbf{y}, \mathbf{a})
\]

Thus, the conditional distribution:
\[
P^\text{train}(\mathbf{y}|\mathbf{a}) = \frac{P^\text{train}(\mathbf{y}, \mathbf{a})}{P^\text{train}(\mathbf{a})} = \frac{P^\text{test}(\mathbf{y}, \mathbf{a})}{P^\text{test}(\mathbf{a})} = P^\text{test}(\mathbf{y}|\mathbf{a})
\]
\end{proof}

\begin{proposition}
Under concept shift-A, the dependencies between $\mathbf{x}$ and $\mathbf{a}$ shifts (\textit{e.g.} $P^\text{train}(\mathbf{a}|\mathbf{x}) \neq P^\text{test}(\mathbf{a}|\mathbf{x})$ or $P^\text{train}(\mathbf{x}|\mathbf{a}) \neq P^\text{test}(\mathbf{x}|\mathbf{a})$), and the conditional distribution of $\mathbf{y}$ given $A$ shifts (\textit{e.g.} $P^\text{train}(\mathbf{y}|\mathbf{a}) \neq P^\text{test}(\mathbf{y}|\mathbf{a})$). Consequently, the conditional distribution of $Y$ given both $\mathbf{x}$ and $\mathbf{a}$ shifts (\textit{e.g.} $P^\text{train}(\mathbf{y}|\mathbf{x},\mathbf{a}) \neq P^\text{test}(\mathbf{y}|\mathbf{x},\mathbf{a})$). However, the conditional distribution of $\mathbf{y}$ given $\mathbf{x}$ alone remain invariant (\textit{e.g.}  $P^\text{train}(\mathbf{y}|\mathbf{x}) = P^\text{test}(\mathbf{y}|\mathbf{x})$). 
\end{proposition}

\begin{proof}
Under concept shift-A, we have:
\[
P^\text{train}(\mathbf{z}) = P^\text{test}(\mathbf{z})
\]
while the transformation matrices and bias vectors for generating \(\mathbf{a}\) change, but those for \(\mathbf{x}\) and \(\mathbf{y}\) remain the same between training and testing datasets:
\[
\mathcal{M}_{x}^\text{train} = \mathcal{M}_{x}^\text{test}, \quad \mathbf{b}_{x}^\text{train} = \mathbf{b}_{x}^\text{test}
\]
\[
\mathcal{M}_{y}^\text{train} = \mathcal{M}_{y}^\text{test}, \quad \mathbf{b}_{y}^\text{train} = \mathbf{b}_{y}^\text{test}
\]
\[
\mathcal{M}_{s}^\text{train} \neq \mathcal{M}_{s}^\text{test}, \quad \mathcal{M}_{o}^\text{train} \neq \mathcal{M}_{o}^\text{test}
\]

The generation processes are:
\[
\mathbf{x}_i = \mathcal{M}_{x} \mathbf{z}_i + \mathbf{b}_{x}
\]
\[
\mathbf{y}_i = \mathcal{M}_{y} \mathbf{z}_i + \mathbf{b}_{y}
\]
\[
p^\text{train}(A_{ij} = 1) = c \left( \| \mathcal{M}_{s}^\text{train} \mathbf{z}_i - \mathcal{M}_{o}^\text{train} \mathbf{z}_j \|^2_2 + 1 \right)^{-1}
\]
\[
p^\text{test}(A_{ij} = 1) = c \left( \| \mathcal{M}_{s}^\text{test} \mathbf{z}_i - \mathcal{M}_{o}^\text{test} \mathbf{z}_j \|^2_2 + 1 \right)^{-1}
\]
where \(\mathbf{a}_i\) is defined as:
\[
\mathbf{a}_i = \text{agg}(\{ \mathcal{M}_a \mathbf{z}_j \}) \quad \text{where} \quad j \in \mathcal{N}_i^1 \cup \mathcal{N}_i^2 \cup \ldots \cup \mathcal{N}_i^l
\]

Given that \(P(\mathbf{z})\) remains the same, the marginal distribution of \(\mathbf{x}\) does not change:
\[
P^\text{train}(\mathbf{x}) = P^\text{test}(\mathbf{x})
\]

However, the change in the transformation matrices \(\mathcal{M}_{s}\) and \(\mathcal{M}_{o}\) implies that the marginal distribution of \(\mathbf{a}\) changes:
\[
P^\text{train}(\mathbf{a}) \neq P^\text{test}(\mathbf{a})
\]

Since \(\mathbf{a}\) is an aggregation of transformed neighbors' latent vectors \(\mathbf{z}_j\), and the generation process for \(\mathbf{a}\) changes, the joint distribution \(P(\mathbf{x}, \mathbf{a})\) changes due to the change in \(P(\mathbf{a})\):
\[
P^\text{train}(\mathbf{x}, \mathbf{a}) \neq P^\text{test}(\mathbf{x}, \mathbf{a})
\]

This implies that:
\[
P^\text{train}(\mathbf{a}|\mathbf{x}) = \frac{P^\text{train}(\mathbf{x}, \mathbf{a})}{P^\text{train}(\mathbf{x})} \neq \frac{P^\text{test}(\mathbf{x}, \mathbf{a})}{P^\text{test}(\mathbf{x})} = P^\text{test}(\mathbf{a}|\mathbf{x})
\]
\[
P^\text{train}(\mathbf{x}|\mathbf{a}) = \frac{P^\text{train}(\mathbf{x}, \mathbf{a})}{P^\text{train}(\mathbf{a})} \neq \frac{P^\text{test}(\mathbf{x}, \mathbf{a})}{P^\text{test}(\mathbf{a})} = P^\text{test}(\mathbf{x}|\mathbf{a})
\]

The generation process for \(\mathbf{y}\) is:
\[
\mathbf{y}_i = \mathcal{M}_{y} \mathbf{z}_i + \mathbf{b}_{y}
\]

Since \(\mathcal{M}_{y}\) and \(\mathbf{b}_{y}\) are the same across training and testing, and \(\mathbf{z}_i\) influences \(\mathbf{a}_i\) through \(\mathcal{M}_{s}\) and \(\mathcal{M}_{o}\), the change in \(\mathcal{M}_{s}\) and \(\mathcal{M}_{o}\) affects the joint distribution \(P(\mathbf{y}, \mathbf{a})\):
\[
P^\text{train}(\mathbf{y}, \mathbf{a}) \neq P^\text{test}(\mathbf{y}, \mathbf{a})
\]

Thus, the conditional distribution:
\[
P^\text{train}(\mathbf{y}|\mathbf{a}) = \frac{P^\text{train}(\mathbf{y}, \mathbf{a})}{P^\text{train}(\mathbf{a})} \neq \frac{P^\text{test}(\mathbf{y}, \mathbf{a})}{P^\text{test}(\mathbf{a})} = P^\text{test}(\mathbf{y}|\mathbf{a})
\]

Since \(\mathbf{y}\) and \(\mathbf{x}\) are both derived from \(\mathbf{z}\) through invariant transformation matrices and bias vectors, and \(P(\mathbf{z})\) remains the same, the joint distribution \(P(\mathbf{y}, \mathbf{x})\) remains unchanged:
\[
P^\text{train}(\mathbf{y}, \mathbf{x}) = P^\text{test}(\mathbf{y}, \mathbf{x})
\]

Thus, the conditional distribution:
\[
P^\text{train}(\mathbf{y}|\mathbf{x}) = \frac{P^\text{train}(\mathbf{y}, \mathbf{x})}{P^\text{train}(\mathbf{x})} = \frac{P^\text{test}(\mathbf{y}, \mathbf{x})}{P^\text{test}(\mathbf{x})} = P^\text{test}(\mathbf{y}|\mathbf{x})
\]

The change in the joint distribution \(P(\mathbf{x}, \mathbf{a})\) implies a change in the joint distribution \(P(\mathbf{y}, \mathbf{x}, \mathbf{a})\) since \(\mathbf{y}\) depends on both \(\mathbf{x}\) and \(\mathbf{a}\):
\[
P^\text{train}(\mathbf{y}, \mathbf{x}, \mathbf{a}) \neq P^\text{test}(\mathbf{y}, \mathbf{x}, \mathbf{a})
\]

Thus, the conditional distribution:
\[
P^\text{train}(\mathbf{y}|\mathbf{x}, \mathbf{a}) = \frac{P^\text{train}(\mathbf{y}, \mathbf{x}, \mathbf{a})}{P^\text{train}(\mathbf{x}, \mathbf{a})} \neq \frac{P^\text{test}(\mathbf{y}, \mathbf{x}, \mathbf{a})}{P^\text{test}(\mathbf{x}, \mathbf{a})} = P^\text{test}(\mathbf{y}|\mathbf{x}, \mathbf{a})
\]

\end{proof}

\section{Important assumptions for CATE estimation}
\label{appendix-CATE-assumptions}
\begin{assumption}
(SUTVA). The potential outcomes of any unit do not vary with the treatment assigned to other units, and, for each unit, there are no different forms or versions of each treatment level, which leads to different potential outcomes. 
\end{assumption}

We discuss the reasonableness of this assumption in section \ref{sec:graph-generation-process}.

\begin{assumption}
(Consistency). The potential outcome of treatment $t$ equals the observed outcome if the actual treatment received is $t$. 
\end{assumption}

\begin{assumption}
(Ignorability). Given pretreatment covariate $\mathbf{X}$, the outcome variable $Y_0$ and $Y_1$ is independent of treatment assignment, i.e. $(Y_{0}, Y_{1}\indep T|X)$.  
\end{assumption}
This assumption is also called ``\textit{no unmeasured confounder}''. This assumption is automatically satisfied with the \textit{``close-world assumption''} made in learning a machine learning model, which implicitly assumes that the input data encompasses the necessary information for making accurate predictions, as we explain in section \ref{sec:graph-generation-process}. In our case, it implies that no other confounders besides $\mathbf{a}$ and $\mathbf{x}$ that affect the output should exist.  

\begin{assumption}
(Positivity). For any set of covariates $\mathbf{x}$, the probability to receive any treatment $t$ is positive, i.e., $0<P(T=t|X=x)<1, \forall t, x$.
\end{assumption}

\section{Derivation of the prediction function of SGC}\label{append:sgc}
A typical SGC makes predictions with the classifier:

\begin{equation}
\label{eqn:sgc}
    \hat{y} = \sigma ({S}^{k}\mathbf{X}\Theta),
\end{equation}
where $S$ is the
“normalized” adjacency matrix $S = \tilde{D}^{-\frac{1}{2}}\tilde{A}\tilde{D}^{\frac{1}{2}}$, $\tilde{A} = A+I$, $\tilde{D}$ is the degree matrix of $\tilde{A}$, $k$ is the number of layers, $\Theta$ is the parameterized weights of each layer into a single matrix: $\Theta = \Theta_{0}\Theta_{1}...\Theta_{k}$. 
Note that the above equation \ref{eqn:sgc} averages the representation of all nodes at each hop, so the effect of the central node is diminished when its neighbor size is large. Alternatively, we can assign a fixed weight $0<\gamma<1$ to the central node, and the rest  $1-\mathbf{\gamma}$ is shared by the neighboring nodes during the aggregation process, so the $S^k$ in equation  \ref{eqn:sgc} is replaced by $S'$, where:
\begin{equation}
S' = \tilde{D}_{1}^{-\frac{1}
{2}}A\tilde{A}^k\tilde{D}_{1}^{\frac{1}{2}} + \tilde{D}_{2}^{-\frac{1}
{2}}I\tilde{D}_{2}^{\frac{1}{2}},
\end{equation}
where $\tilde{D}_1$ and $\tilde{D}_2$ are diagonal matrix such that $\tilde{D}_1(i,i) = (1-\gamma) D^k(i,i)$ and $\tilde{D}_2(i,i) = \gamma D^k(i,i)$. The new prediction function is then expressed as:
\begin{align}
    \hat{\mathbf{y}_i} =&\sigma(\mathbf{h}^y_i) =\sigma ([\mathbf{S'}^{k}\mathbf{X}\Theta]_i)\\
\label{eq::predict_xunc_1}
=&\sigma\Bigl(\underbrace{[\tilde{D}_{1}^{-\frac{1}
{2}}A\tilde{A}^k\tilde{D}_{1}^{\frac{1}{2}}\mathbf{X}\Theta]_i}_{\Psi'_{\mathbf{a}}(\mathbf{a}_i)} + \underbrace{[(\tilde{D}_{2}^{-\frac{1}
{2}}I\tilde{D}_{2}^{\frac{1}{2}})^{k}\mathbf{X}\Theta]_i}_{\Psi'_{\mathbf{x}}(\mathbf{x_i})}\Bigr). 
\end{align}
Inside $\sigma(\cdot)$, the first term $\tilde{D}_{1}^{-\frac{1}
{2}}A\tilde{A}^k\tilde{D}_{1}^{\frac{1}{2}}X\Theta$ models the contribution of the neighborhood nodes' representation (excluding central node) on $\mathbf{h}^{y}_i$,  we call it $\Psi'_{\mathbf{a}}(\mathbf{a}_i)$; similarly, the second term $(\tilde{D}_{2}^{-\frac{1}
{2}}I\tilde{D}_{2}^{\frac{1}{2}})^{k}X\Theta$ models the contribution of the features of the central node and we call it $\Psi'_{\mathbf{x}}(\mathbf{x}_i)$. 
Note that the magnitudes of the diagnal matrixes of $\Psi'_{\mathbf{a}}(\mathbf{a}_i)$ and $\Psi'_{\mathbf{x}}(\mathbf{x}_i)$ are scaled by the parameter $\gamma$. We can further rewrite Equation \ref{eq::predict_xunc_1} in the unscaled form:

\begin{equation}
\sigma\Bigl(\gamma\underbrace{[({{\tilde{D}}^{k}})^{-\frac{1}
{2}}A\tilde{A}^k({{\tilde{D}}^{k}})^{\frac{1}{2}}\mathbf{X}\Theta]_i}_{\Psi_{\mathbf{a}}(\mathbf{a}_i)} + (1-\gamma)\underbrace{[(({{\tilde{D}}^{k}})^{-\frac{1}
{2}}I({{\tilde{D}}^{k}})^{\frac{1}{2}})^{k}\mathbf{X}\Theta]_i}_{\Psi_{\mathbf{x}}(\mathbf{x_i})}\Bigr), 
\end{equation}



where $[({{\tilde{D}}^{k}})^{-\frac{1}
{2}}A\tilde{A}^k({{\tilde{D}}^{k}})^{\frac{1}{2}}\mathbf{X}\Theta]_i$ can be viewed as $\Psi_{\mathbf{a}}(\mathbf{a}_i)$, and $[(({{\tilde{D}}^{k}})^{-\frac{1}
{2}}I({{\tilde{D}}^{k}})^{\frac{1}{2}})^{k}\mathbf{X}\Theta]_i$ can be viewed as $\Psi_{\mathbf{x}}(\mathbf{x_i})$, which aligns with asssumtion \ref{assumption-2}. 

\section{Dual Casual Decomposition}
\label{appendix:dual}
We aim to learn two sets of $g(\cdot), h(\cdot)$ for the two casual models SCM-F and SCM-A. We denote them as $g^{X}(\cdot), h^{X}(\cdot)$ and $g^{A}(\cdot), h^{A}(\cdot)$, separately. Each casual model is also associated with a set of $e(\cdot), m(\cdot)$, denoted as $e^{X}(\cdot), m^{X}(\cdot)$ and $e^{A}(\cdot), m^{A}(\cdot)$.
Unlike SIN, which learns all model parameters within the two-stage training procedure, we allow the model to learn  $g^{X}(\cdot)$, $h^{A}(\cdot)$, and $m^{X}(\cdot)$ with shared parameters beforehand. 
Since $g^{X}(\mathbf{a})$, $h^{A}(\mathbf{a})$, and $m^{X}(\mathbf{a})$ are all functions of the neighborhood representation $\mathbf{a}$, whose value is determined with a $L$-layer GNN model $p^{GNN}(\cdot)$ that generates a neighborhood representation by aggregating the embedding of nodes in $L$-hop neighborhood without including the central node. 

We then map the GNN embedding to the space of $\mathbf{h}^y$ with an MLP layer $p'(\cdot)$ that follows $p(\cdot)$. 

The parameters of $p(\cdot)$ and $p'(\cdot)$, $\rho, \rho'$, are learned with the goal of minimizing:
        $J_{p, p'}(\rho, \rho') = \frac{1}{n} \sum_{i=1}^{n}\mathcal{L}_\text{mse}\left(\mathbf{h}^y_i, \Bigr[ \hat{p'}_{\rho'}\bigl(\hat{p}_{\rho} (\mathbf{X}, A)\bigr)
        \Bigr]_i \right)$.
As the ground truth $\mathbf{h}^y$ is not available, while $\mathbf{y}=\sigma(\mathbf{h}^y)$ is available, we thus apply $\sigma(\cdot)$ on both sides and minimizing the following cross entropy function instead: 
\begin{equation}
        J_{p, p'}(\rho, \rho') = \frac{1}{n} \sum_{i=1}^{n}\mathcal{L}_{ce}\left(\mathbf{y}_i, \sigma \Bigr( \Bigr[ \hat{p'}_{\rho'}\bigl(\hat{p}_{\rho} (\mathbf{X}, A)\bigr) \Bigr]_i \Bigr) \right).
\label{J_pp}
\end{equation}
We apply this alternation for the rest of the training process. 
With the optimized $p(\cdot)$, we first assign the learned neighborhood representation to $\mathbf{a}$, such that  $\mathbf{a}_i \triangleq [\hat{p}_{\rho} (\mathbf{X}, A)]_i$. 
Without losing generalizability, we assign $g^{X}(\mathbf{a})$, $h^{A}(\mathbf{a})$ as the same as $\mathbf{a}$. We then estimate $m^{X}(\mathbf{a})$, with the optimized $p'(\cdot)$.  $g^{X}(\mathbf{a})$, $h^{A}(\mathbf{a})$, $m^{X}(\mathbf{a})$ and  remain fixed values in the rest of the learning process: $g^{X}(\mathbf{a}_i) \triangleq h^{A}(\mathbf{a}_i) \triangleq
    \mathbf{h}_i^{\mathbf{a}} \triangleq \mathbf{a}_{i},     m^{X}(\mathbf{a}_i) \triangleq \mathbf{m}_i^{\mathbf{a}} \triangleq  \hat{p'}_{\rho'}\bigl(\mathbf{a}_{i}\bigr)$

We learn the remaining parameters for each casual model. For SCM-A, we follow the two-stage procedure:

\textbf{Stage 1:} Learn parameter $\theta^A$ of $\hat{m}_{\theta}(\mathbf{x})$ {to minimize the cross-entropy loss as following:} 
$
        J_{m}(\theta^A) = \frac{1}{n} \sum_{i=1}^{n} \mathcal{L}_{ce}\left(\mathbf{y}_{i}, \sigma \bigl( \hat{m}_{\theta^A}(\mathbf{x}_{i}) \bigr) \right)^2$

\textbf{Stage 2:} Learn parameter $\psi^A$ for $g^{A}(\cdot)$ and $\eta^A$ for $e^{A}(\cdot)$ with the objectives:

\begin{equation}
        J_{g}(\psi^A) = \frac{1}{n} \sum_{i=1}^{n} \mathcal{L}_{ce} \Biggl( \mathbf{y}_{i}, \sigma \biggl( \hat{m}_{\theta^A}(\mathbf{x}_{i}) +  \hat{g}_{\psi^A}(\mathbf{x}_{i})^{\top} \bigl( \mathbf{h}_i^{\mathbf{a}} - \hat{e}_{\eta^A}^A (\mathbf{x}_i)\bigr) \biggr)\Biggr)^2, 
        J_{e}(\eta^A) = \frac{1}{n} \sum_{i=1}^{n} \left\|\mathbf{h}_i^{\mathbf{a}} - \hat{e}_{\eta^A}^A (\mathbf{x}_i) \right\| ^2_2
\end{equation}
For SCM-F, stage 1 is no longer necessary as $\hat{m}_{\theta}(\mathbf{a}_i)$ is fixed as $\mathbf{m}^{\mathbf{a}}_i$. We only need to learn parameter $\phi^X$ for $h^{X}(\cdot)$ and $\eta_F$ for $e^{X}(\cdot)$ with the objectives:
\begin{equation}
        J_{h}(\phi^X) = \frac{1}{n} \sum_{i=1}^{n} \mathcal{L}_{ce} \Biggl( \mathbf{y}_{i}, \sigma \biggl( \mathbf{m}^{\mathbf{a}}_i +  {\mathbf{h}_i^{\mathbf{a}}}^{\top} \bigl( \hat{h}^X_{\phi^X}(\mathbf{x}_i) - \hat{e}_{\eta^X}^A (\mathbf{a}_i)\bigr) \biggr)\Biggr)^2, 
        J_{e}(\eta^X) = \frac{1}{n} \sum_{i=1}^{n} \left\|\mathbf{h}^A - \hat{e}_{\eta^X}^A ({\mathbf{h}_i^{\mathbf{a}}}) \right\| ^2_2
\end{equation}

We follow the alternating optimization process in \cite{kaddour2021causal} which updates $\psi, \phi$ more frequently than $\eta$ to achieve a more stabilized training process. 

\section{Complexity analysis}
\label{sec::compare}

Consider a graph with \( n \) nodes and  \( e \) edges, and an average degree \( \bar{d} \). A Graph Neural Network (GNN) with \( L \) layers computes embeddings with a time and space complexity of \( O\left(n L \bar{d}^{2}\right) \). When obtaining the GNN embedding, \model{} performs one encoder computation per update step. During training, five distinct encoders are learned for causal models, each with a time complexity of \( O(n) \) per update step. Therefore, the overall time complexity is \( O\left(n L \bar{d}^{2}\right) + O\left(5n\right) \).

We compare the complexity and requirements of \model{} with other methods in Table \ref{tab::compare}. For EERM and FLOOD, $K$ is the number of augmented training environments. For CIT, $K$ is the number of clusters and $p$ is the probability of transfer. As it shows,  \model{} has competitive complexity, while having the least restrictive requirements, making it applicable to a wide range of scenarios.

\begin{table}[ht]
    \caption{Comparision between methods.\newline}
    \centering
    \tiny
    \begin{tabular}{ c| p{10mm} p{15mm} p{15mm} p{15mm} p{15mm} c}
    \toprule
     Method&Tailored \newline for graphs & Multiple \newline training envs & Training envs \newline augmentation & Access to test  \newline distributions & Test-time 
     \newline training & Complexity \\
    \midrule
       ERM&N/A&N/A&N/A&N/A&N/A&\( O\left(n L \bar{d}^{2}\right) \)\\
       IRM&\ding{55}&\textcolor{red}{Required}&{Not Required}&{Not Required}&{Not Required}&\( O\left(n L \bar{d}^{2}\right) \)\\
       REX&\ding{55}&\textcolor{red}{Required}&{Not Required}&{Not Required}&{Not Required}&\( O\left(n L \bar{d}^{2}\right) \)\\
       EERM&\ding{51}&{Not Required}&\textcolor{red}{Required}&{Not Required}&{Not Required}& $O\left(K\left(n L \bar{d}^{2}\right)\right)$\\
       CIT&\ding{51}&{Not Required}&{Not Required}&{Not Required}&{Not Required}& $O\left(nL\bar{d}^{2}\right)+O(K(e+nK)+pn)$\\
       FLOOD&\ding{51}&\textcolor{red}{Required}&{Not Required}&{Not Required}&\textcolor{red}{ Required}& $O\left((K+2)\left(n L \bar{d}^{2}\right)\right)$\\
       SR-GNN&\ding{51}&{Not Required}&{Not Required}&\textcolor{red}{Required}&{ Not Required}&\( O\left(nL\bar{d}^{2}+n^{2}\right) \)\\
       \model{}&\ding{51}&{Not Required}&{Not Required}&{Not Required}&{Not Required}& \( O\left(n L \bar{d}^{2}\right) + O\left(5n\right) \)\\
       \bottomrule
    \end{tabular}
    \label{tab::compare}
\end{table}

\section{Datasets and Setup}\label{append:datasetup}

\subsection{Hyperparameter Setup}\label{append:hyper}

The hidden size of the backbone GNNs of all methods is searched from {8, 16, 32, 64, 128}, the number of heads for GAT is searched from {4, 8}, and the number of layers is 2. We use Adam as the optimizer with a learning rate of 1e-3 and weight decay of 1e-5.
For methods with penalty weights, we searched from different values centered on their default value. For instance, for IRM, the default penalty weight is 1e5, we then conduct our search on {1e2, 1e3, 1e4, 1e5, 1e6, 1e7, 1e8}. We follow the default setting for other hyperparameters, such as the number of augmented views. All experiments are conducted on an NVIDIA GeForce RTX 3090 GPU with 24GB memory.

\subsection{Soft label leave-out setting}
In our experiments, if we have 6 classes and the training set has $80\%$ from the first two classes, $10\%$ from the second two classes, and $10\%$ from the last two classes. Then, the validation set owns $10\%$ from the first two classes, $80\%$ from the second two classes, and $10\%$ from the last two classes.  Test set owns $10\%$ from the first two classes, $10\%$ from the second two classes, and $80\%$ from the last two classes.

\begin{table}[ht]
    \caption{Statistics of real-world single graph datasets.\newline}
    \centering
    \scriptsize
    \begin{tabular}{c c c c c c c c }
    \toprule
     & \texttt{Cora} & \texttt{Citeseer} & \texttt{Amazon-photo} &\texttt{Coauthor-CS} & \texttt{Squirrel} & \texttt{Roman-empire} & \texttt{Tolokers}\\
    \midrule
       \# Node &2,708&3,327&7,650&18,333&2,223&22,662&11,758\\
       \# Edge &5,278&4,552&119,081&81,894&23,499&32,927&259,500\\
       \# Class &7&6&8&15&5&18&2\\
       \# Feat&1433&3703&745&6,805&2,089&300&10\\
       Metric&Marco-F1&Marco-F1&Marco-F1&Marco-F1&Marco-F1&Marco-F1&F1 score\\
       \bottomrule
    \end{tabular}
    \label{tab:stat:single-graph}
\end{table}
  
\begin{table}[ht]
    \caption{Statistics of the \texttt{Facebook-100} dataset.\newline}
    \centering
    \scriptsize
    \begin{tabular}{l  c c c c c c c }
    \toprule
     & \texttt{Johns Hopkins} & \texttt{Caltech} & \texttt{Amherst} &\texttt{Bingham} & \texttt{Duke} & \texttt{Princeton}&\texttt{WashU} \\

    \midrule
       \# Node & 5,180 & 769&2,235&10,004&9,895&6,596& 7,755\\
       \# Edge & 373,172&33,312&181,908&725,788&1,012,884&586,640&735,082\\
       Positive rate & 43\% & 53\%&36\%&40\%&39\%&37\%&38\%\\
    \midrule
    & \texttt{Brandeis} & \texttt{Carnegie} & \texttt{Penn} & \texttt{Brown} & \texttt{Texas} &\texttt{Cornell5}, &\texttt{Yale}\\
    \midrule
        \# Node &3,898&6,637&41,554&8,600&31,560&18,660&8,578\\
       \# Edge &275,134&499,934&2,724,458&769,052&2,439,300&1,581,554&810,900\\
       Positive rate &30\%&47\%&43\%&32\%&37\%&37\%&35\%\\
       \bottomrule
    \end{tabular}
    \label{tab:stat:fb100}
\end{table}
\begin{table}[ht]
    \centering
        \caption{Statistics of the \texttt{OGB-elliptic} dataset.\newline}
    \scriptsize
    \begin{tabular}{c c c c c c c c}
    \toprule
    Time slot & 1-6 & 7-12 & 13-18 & 19-24 & 25-30 &31-36& 37-43\\
    \midrule
       \# Node & 28,571 & 18,525 & 25,985 & 14,337 & 24,878 & 25,920 & 29,684\\
       \# Edge &33,835 &19,613 & 29,274 & 15,296 & 28,223 & 29,689 & 33,659 \\
       Positive rate &11\%&22\%&12\%&23\%&12\%&10\%&3\%\\
       \bottomrule
    \end{tabular}
   \label{tab:stat:OG}
\end{table}
\begin{table}[ht]
    \caption{Statistics of synthetic datasets.\newline}
    \centering
    \scriptsize
    \begin{tabular}{c c c c }
    \toprule
     & \texttt{h-feat} & \texttt{qtr-feat} & \texttt{full-feat}\\
    \midrule
       \# Node & 8,000& 8,000&8,000 \\
       \# Edge & 404,597 & 1,487,637 &35,850\\
       \# Class &4&4&4 \\
       \# Feat&8&4&16\\
       \bottomrule
    \end{tabular}
   \label{tab:stat:syn}
\end{table}
\begin{table}[ht]
\caption {Test F1 scores on \texttt{Facebook-100} dataset. The results from the GNN with the highest validation F1 score are reported. The best results are bold-faced.\newline}
\tiny
\centering
\begin{tabular}
{ l cc cc cc cc cc cc cc cc cc}
 \toprule
 Training&\multicolumn{6}{c}{\texttt{Johns Hopkins} + \texttt{Caltech} + \texttt{Amherst}} &  \multicolumn{6}{c}{\texttt{Bingham} + \texttt{Duke} + \texttt{Princeton}} & 
  \multicolumn{6}{c}{\texttt{WashU} + \texttt{Brandeis} + \texttt{Carnegie}} \\
 
  Test& \multicolumn{2}{c}{\texttt{Penn}} &  \multicolumn{2}{c}{\texttt{Brown}} & \multicolumn{2}{c}{\texttt{Texas}} & \multicolumn{2}{c}{\texttt{Penn}} &  \multicolumn{2}{c}{\texttt{Brown}} & \multicolumn{2}{c}{\texttt{Texas}} & \multicolumn{2}{c}{\texttt{Penn}} &  \multicolumn{2}{c}{\texttt{Brown}} & \multicolumn{2}{c}{\texttt{Texas}} \\ 
  \midrule
  ERM&\multicolumn{2}{c}{49.23$\pm$1.72}&\multicolumn{2}{c}{49.68$\pm$0.93}&\multicolumn{2}{c}{48.57$\pm$0.21}&\multicolumn{2}{c}{51.42$\pm$4.25}&\multicolumn{2}{c}{51.45$\pm$1.48}&\multicolumn{2}{c}{47.37$\pm$4.78}&\multicolumn{2}{c}{47.34$\pm$5.48}&\multicolumn{2}{c}{48.08$\pm$2.51}&\multicolumn{2}{c}{48.36$\pm$4.30}\\
  IRM&\multicolumn{2}{c}{35.26$\pm$2.40}&\multicolumn{2}{c}{46.92$\pm$5.66}&\multicolumn{2}{c}{36.86$\pm$1.64}&\multicolumn{2}{c}{42.12$\pm$1.99}&\multicolumn{2}{c}{51.34$\pm$0.90}&\multicolumn{2}{c}{41.57$\pm$4.31}&\multicolumn{2}{c}{50.16$\pm$1.30}&\multicolumn{2}{c}{49.62$\pm$3.32}&\multicolumn{2}{c}{46.41$\pm$5.28}\\
  REX&\multicolumn{2}{c}{44.77$\pm$6.48}&\multicolumn{2}{c}{42.65$\pm$7.34}&\multicolumn{2}{c}{44.05$\pm$8.88}&\multicolumn{2}{c}{43.77$\pm$5.72}&\multicolumn{2}{c}{47.26$\pm$5.75}&\multicolumn{2}{c}{44.36$\pm$7.82}&\multicolumn{2}{c}{39.67$\pm$8.59}&\multicolumn{2}{c}{44.65$\pm$6.67}&\multicolumn{2}{c}{40.28$\pm$7.59}\\
  EERM&\multicolumn{2}{c}{22.62$\pm$22.91}&\multicolumn{2}{c}{49.44$\pm$1.92}&\multicolumn{2}{c}{49.12$\pm$1.71}&\multicolumn{2}{c}{18.91$\pm$18.99}&\multicolumn{2}{c}{45.95$\pm$3.74}&\multicolumn{2}{c}{47.83$\pm$1.17}&\multicolumn{2}{c}{24.40$\pm$24.62}&\multicolumn{2}{c}{47.58$\pm$2.91}&\multicolumn{2}{c}{51.27$\pm$1.04}\\
  CIT&\multicolumn{2}{c}{44.66$\pm$6.65}&\multicolumn{2}{c}{45.26$\pm$6.21}&\multicolumn{2}{c}{42.10$\pm$8.97}&\multicolumn{2}{c}{45.82$\pm$6.46}&\multicolumn{2}{c}{48.62$\pm$1.88}&\multicolumn{2}{c}{39.79$\pm$4.74}&\multicolumn{2}{c}{37.99$\pm$6.54}&\multicolumn{2}{c}{39.66$\pm$6.58}&\multicolumn{2}{c}{39.88$\pm$6.27}\\
  FLOOD&\multicolumn{2}{c}{42.37$\pm$5.06}&\multicolumn{2}{c}{41.48$\pm$5.28}&\multicolumn{2}{c}{40.82$\pm$5.94}&\multicolumn{2}{c}{46.99$\pm$7.48}&\multicolumn{2}{c}{47.28$\pm$5.61}&\multicolumn{2}{c}{44.48$\pm$5.22}&\multicolumn{2}{c}{41.21$\pm$7.68}&\multicolumn{2}{c}{46.24$\pm$8.52}&\multicolumn{2}{c}{40.16$\pm$6.01}\\

  StableGL&\multicolumn{2}{c}{44.54$\pm$6.58}&\multicolumn{2}{c}{45.64$\pm$6.25}&\multicolumn{2}{c}{43.75$\pm$5.65}&\multicolumn{2}{c}{46.94$\pm$8.54}&\multicolumn{2}{c}{47.77$\pm$4.79}&\multicolumn{2}{c}{47.51$\pm$5.94}&\multicolumn{2}{c}{38.31$\pm$8.42}&\multicolumn{2}{c}{43.22$\pm$4.50}&\multicolumn{2}{c}{37.86$\pm$4.75}\\
\model&\multicolumn{2}{c}{\textbf{55.31$\pm$0.40}}&\multicolumn{2}{c}{\textbf{53.31$\pm$0.11}}&\multicolumn{2}{c}{\textbf{53.56$\pm$0.19}}&\multicolumn{2}{c}{\textbf{54.59$\pm$0.35}}&\multicolumn{2}{c}{\textbf{53.48$\pm$0.15}}&\multicolumn{2}{c}{\textbf{53.12$\pm$0.19}}&\multicolumn{2}{c}{\textbf{54.44$\pm$1.18}}&\multicolumn{2}{c}{\textbf{53.02$\pm$0.36}}&\multicolumn{2}{c}{\textbf{53.05$\pm$0.86}}\\
 
  \bottomrule
\end{tabular}
\label{tab::results-facebook-complete}
\end{table} 
\begin{table}
\caption {Test Macro-F1 scores on synthetic datasets. The best results are bold-faced.\newline}
\tiny
\centering
\begin{tabular}
{ l l cc cc cc}
 \toprule
  && \multicolumn{2}{c}{\texttt{h-feat}} &  \multicolumn{2}{c}{\texttt{qrt-feat}} & \multicolumn{2}{c}{\texttt{full-feat}}\\ 
  \midrule
\multirow{7}{*}{\rotatebox[origin=c]{90}{SGC}}& ERM & \multicolumn{2}{c}{47.41$\pm$0.25} & \multicolumn{2}{c}{48.62$\pm$1.88} & \multicolumn{2}{c}{48.57$\pm$0.21} \\
& IRM & \multicolumn{2}{c}{33.78$\pm$1.41} & \multicolumn{2}{c}{33.36$\pm$0.70} & \multicolumn{2}{c}{33.04$\pm$1.40} \\
& EERM & \multicolumn{2}{c}{38.14$\pm$8.25} & \multicolumn{2}{c}{37.86$\pm$5.15} & \multicolumn{2}{c}{39.74$\pm$7.09} \\
& CIT & \multicolumn{2}{c}{21.11$\pm$21.33} & \multicolumn{2}{c}{36.41$\pm$0.45} & \multicolumn{2}{c}{49.12$\pm$1.71} \\
& REX & \multicolumn{2}{c}{32.87$\pm$1.12} & \multicolumn{2}{c}{33.85$\pm$0.50} & \multicolumn{2}{c}{34.47$\pm$0.82} \\
& FLOOD & \multicolumn{2}{c}{39.53$\pm$8.83} & \multicolumn{2}{c}{37.26$\pm$4.82} & \multicolumn{2}{c}{39.13$\pm$6.39} \\
& StableGL & \multicolumn{2}{c}{41.79$\pm$8.86} & \multicolumn{2}{c}{38.30$\pm$5.60} & \multicolumn{2}{c}{41.22$\pm$6.30} \\
& \textbf{DeCaf} & \multicolumn{2}{c}{\textbf{55.31$\pm$0.40}} & \multicolumn{2}{c}{\textbf{53.16$\pm$0.17}} & \multicolumn{2}{c}{\textbf{53.56$\pm$0.19}} \\
  \midrule
  \multirow{7}{*}{\rotatebox[origin=c]{90}{GCN}}& ERM & \multicolumn{2}{c}{43.12$\pm$1.84} & \multicolumn{2}{c}{49.68$\pm$0.93} & \multicolumn{2}{c}{43.79$\pm$1.36} \\
& IRM & \multicolumn{2}{c}{34.03$\pm$0.94} & \multicolumn{2}{c}{32.90$\pm$0.60} & \multicolumn{2}{c}{32.47$\pm$1.16} \\
& EERM & \multicolumn{2}{c}{43.24$\pm$6.52} & \multicolumn{2}{c}{42.65$\pm$7.34} & \multicolumn{2}{c}{38.60$\pm$5.26} \\
& CIT & \multicolumn{2}{c}{22.79$\pm$23.15} & \multicolumn{2}{c}{49.44$\pm$1.92} & \multicolumn{2}{c}{40.81$\pm$5.74} \\
& REX & \multicolumn{2}{c}{33.10$\pm$1.60} & \multicolumn{2}{c}{35.62$\pm$2.60} & \multicolumn{2}{c}{34.49$\pm$0.89} \\
& FLOOD & \multicolumn{2}{c}{38.58$\pm$6.60} & \multicolumn{2}{c}{40.07$\pm$6.76} & \multicolumn{2}{c}{38.58$\pm$6.57} \\
& StableGL & \multicolumn{2}{c}{42.12$\pm$6.44} & \multicolumn{2}{c}{44.77$\pm$4.78} & \multicolumn{2}{c}{43.75$\pm$5.65} \\
& \textbf{DeCaf} & \multicolumn{2}{c}{\textbf{53.57$\pm$0.98}} & \multicolumn{2}{c}{\textbf{53.31$\pm$0.11}} & \multicolumn{2}{c}{\textbf{52.16$\pm$0.45}} \\

  \midrule
  \multirow{7}{*}{\rotatebox[origin=c]{90}{GAT}}& ERM & \multicolumn{2}{c}{49.23$\pm$1.72} & \multicolumn{2}{c}{47.34$\pm$1.25} & \multicolumn{2}{c}{46.13$\pm$1.54} \\
& IRM & \multicolumn{2}{c}{35.26$\pm$2.40} & \multicolumn{2}{c}{46.92$\pm$5.66} & \multicolumn{2}{c}{36.86$\pm$1.64} \\
& EERM & \multicolumn{2}{c}{44.77$\pm$6.48} & \multicolumn{2}{c}{41.94$\pm$3.92} & \multicolumn{2}{c}{44.05$\pm$8.88} \\
& CIT & \multicolumn{2}{c}{22.62$\pm$22.91} & \multicolumn{2}{c}{39.04$\pm$4.31} & \multicolumn{2}{c}{34.66$\pm$2.32} \\
& REX & \multicolumn{2}{c}{44.66$\pm$6.65} & \multicolumn{2}{c}{45.26$\pm$6.21} & \multicolumn{2}{c}{42.10$\pm$8.97} \\
& FLOOD & \multicolumn{2}{c}{42.37$\pm$5.06} & \multicolumn{2}{c}{41.48$\pm$5.28} & \multicolumn{2}{c}{40.82$\pm$5.94} \\
& StableGL & \multicolumn{2}{c}{44.54$\pm$6.58} & \multicolumn{2}{c}{45.64$\pm$6.25} & \multicolumn{2}{c}{41.88$\pm$6.88} \\
& \textbf{DeCaf} & \multicolumn{2}{c}{\textbf{51.84$\pm$0.88}} & \multicolumn{2}{c}{\textbf{52.29$\pm$0.56}} & \multicolumn{2}{c}{\textbf{50.68$\pm$0.63}} \\

  \midrule
   \multirow{7}{*}{\rotatebox[origin=c]{90}{H2GCN}}&ERM&\multicolumn{2}{c}{49.38$\pm$3.44}&\multicolumn{2}{c}{31.72$\pm$1.60}&\multicolumn{2}{c}{66.59$\pm$1.86}\\
  &IRM&\multicolumn{2}{c}{51.17$\pm$8.82}&\multicolumn{2}{c}{33.57$\pm$4.78}&\multicolumn{2}{c}{62.04$\pm$1.45}\\
  &EERM&\multicolumn{2}{c}{49.04$\pm$3.54}&\multicolumn{2}{c}{30.04$\pm$1.72}&\multicolumn{2}{c}{65.67$\pm$1.62}\\
  &CIT&\multicolumn{2}{c}{54.33$\pm$3.90}&\multicolumn{2}{c}{30.28$\pm$3.96}&\multicolumn{2}{c}{64.78$\pm$2.64}\\
  &FLOOD&\multicolumn{2}{c}{47.18$\pm$4.97}&\multicolumn{2}{c}{29.56$\pm$1.13}&\multicolumn{2}{c}{64.65$\pm$0.79}\\
&StableGL&\multicolumn{2}{c}{39.99$\pm$2.97}&\multicolumn{2}{c}{37.44$\pm$5.21}&\multicolumn{2}{c}{63.06$\pm$4.10}\\
  &\model{}&\multicolumn{2}{c}{\textbf{55.92$\pm$5.20}}&\multicolumn{2}{c}{\textbf{44.96$\pm$1.78}}&\multicolumn{2}{c}{\textbf{67.02$\pm$1.12}}\\
  \bottomrule
\end{tabular}
\label{tab::appendix-results-synthetic}
\end{table}

\subsection{Real-word datasets}
\label{append:statistics-real-world}
We provide the statistics of single-graph datasets in Table \ref{tab:stat:single-graph}. \texttt{Cora}~\cite{cora}, \texttt{Citeseer}~\cite{citeseer}, and \texttt{Coauthor-CS}~\cite{coauthor-cs} are citation networks. \texttt{Amazon-photo}~\cite{amazon-photo} is a co-purchase network where nodes represent goods for sale on e-commerce websites.  \texttt{Squirrel}, \texttt{Roman-empire}, and \texttt{Tolokers} are heterophilous networks created by Platonov~\cite{platonov2024critical} et al. \texttt{Squirrel} is a Wikipedia network, \texttt{Roman-empire} is created based on the Roman Empire article from English Wikipedia, and \texttt{Tolokers} is created based on data from the Toloka crowdsourcing platform, where the nodes represent tolokers (workers). 


\texttt{Facebook-100}~\cite{fb100} contains 100 social networks collected from universities in the United States. 
Each node represents a student and the goal is to predict the gender of each student. 
 We provide the statistics of sub-datasets we use in Table \ref{tab:stat:fb100}. 
 
 \texttt{OGB-elliptic}~\cite{elliplic} is a dynamic financial network dataset that contains in total of 43 graph snapshots from different time steps.
    Each node represents a Bitcoin transaction, and the goal is to detect illicit transactions. We group all 43 snapshots into 7 timeslots and provide statistics for each timeslot in Table \ref{tab:stat:OG}.

\subsection{Synthetic datasets}
\label{append::synthetic-datasets}
We create three synthetic graphs to simulate the situations where node features and neighborhood representation are dependent or independent of each other. 
For each graph, we randomly sample $n$ instances of $\mathbf{z}$ with 16 features from a multivariate normal distribution. We generate node features, labels, and adjacency matrices based on the data generation process in assumption \ref{assumption_1}. By posing different constraints on $\mathcal{M}_x$, $\mathcal{M}_y$, $\mathcal{M}_s$, and $\mathcal{M}_o$, we can control the dependence/ independence between node features and neighborhood representation, and their contributions to the labels.  Statistical details of these datasets are shown in Table~\ref{tab:stat:syn}.

\texttt{h-feat}: When $\mathcal{M}_x$ can only ``observe'' half of the elements of $\mathbf{z}$ and $\mathcal{M}_{s}, \mathcal{M}_{o}$ can fully observe $\mathbf{z}$, node features and neighborhood representation are correlated with each other, and each of them can reveal extra information about node label. To do so, we assign $\mathcal{M}_x$ as a $16\times8$ matrix with its first $8$ rows as an identical matrix, and the rest rows are all zeros, such that the second half elements of $\mathbf{z}$ do not participate in the construction of $\mathbf{x}$. We assign $\mathcal{M}_s$ and $\mathcal{M}_o$ to be the same size and value as $\mathcal{M}_y$, which is guaranteed to be a non-trivial transformation of $\mathbf{z}$. 

\texttt{qtr-feat}: When $\mathcal{M}_x$ ``observe'' quarter of the elements of $\mathbf{z}$ and $\mathcal{M}_{s}, \mathcal{M}_{o}$ can observe the rest three fourth, node features and neighborhood contributes to the prediction of node labels independently. We assign $\mathcal{M}_x$ as a $16\times4$ matrix with its first $4$ rows as an identical matrix, and the rest rows are all zeros. We assign $\mathcal{M}_s$ and $\mathcal{M}_o$ to be the same size and value as $\mathcal{M}_y$, except with the first $4$ rows replaced by all zeros. 

\texttt{full-feat}: We create another graph where node features and neighborhoods contribute to the prediction of node labels independently in an alternative way. First, we make $\mathcal{M}_x$ fully observe $\mathbf{z}$ by assigning it as a $16\times16$ identical matrix; then by assigning $\mathcal{M}_s$ and $\mathcal{M}_o$ as zeros, we create completely random edges. 
We slightly modify equation~\ref{eqn::y-generation} to be $\mathbf{y}_i = \mathcal{M}_{y}(\frac{1}{2}\mathbf{z}_i + \frac{1}{2}\mathbf{\bar{z}}_i)+ \mathbf{b}_{y}$, where $\bar{\mathbf{z}}_i$ is the mean $\mathbf{z}$ embedding of neighboring nodes of node $v_i$, such that the neighborhood representation can directly affect node labels.

\section{Complete results}
\label{append:complete-results}
We report the complete results on \texttt{facebook-100} when different sets of training graphs are used in Table \ref{tab::results-facebook-complete}.  With different training sets, the proposed \model{} achieves the best classification results over three different test sets compared with state-of-the-art OOD generalization methods.

\section{Limitations}\label{limit}

{
This paper focuses on homogeneous graphs with limited node and edge types. In the future, we plan to extend the method to heterogeneous graphs with more diverse node relations and neighborhood patterns. This extension will broaden the applications of our method to domains such as social networks, healthcare, and biological networks, where heterogeneity can provide rich information for making predictions.}

\section{Broader Impacts}\label{append:sec:bi}
{
\model{} improves the generalizability of the GNN, helping it learn a faithful mapping between inputs and outputs that captures true correlations. This is crucial for critical domains vulnerable to security issues, such as cybersecurity, finance, and healthcare. Learning a robust and generalizable model under potential distribution shifts is essential for these domains.}

\end{document}